\documentclass{article}

 \usepackage[preprint]{neurips_2026}
\usepackage{natbib}

\usepackage[utf8]{inputenc} 
\usepackage[T1]{fontenc}    
\usepackage{url}            
\usepackage{booktabs}       
\usepackage{nicefrac}       
\usepackage{microtype}      
\usepackage{xcolor}         
\usepackage{amsmath,amssymb,amsfonts}
\usepackage{algorithmic}
\usepackage{graphicx}
\usepackage{textcomp}
\usepackage{algorithm}
 \usepackage{subcaption}
 \usepackage{float}
 \usepackage{comment}
 \usepackage{tgpagella}
\usepackage[colorlinks,urlcolor=blue,linkcolor=blue,citecolor=blue]{hyperref}
\usepackage{amsthm}

\newtheorem{proposition}{Proposition}

\def\BibTeX{{\rm B\kern-.05em{\sc i\kern-.025em b}\kern-.08em
    T\kern-.1667em\lower.7ex\hbox{E}\kern-.125emX}}

\title{SAGE: Surrogate-gradient Adaptation via Attention-Guided Entropy for Spiking Transformers}

\author{
Kiran Nair\thanks{Corresponding author: \texttt{kiran.prasannannair@coyotes.usd.edu}}\\
USD Artificial Intelligence Research Lab\\
Department of Computer Science\\
University of South Dakota\\
Vermillion, SD 57069, USA
\And
Rodrigue Rizk\\
USD Artificial Intelligence Research Lab\\
Department of Computer Science\\
University of South Dakota\\
Vermillion, SD 57069, USA
\And
KC Santosh\\
USD Artificial Intelligence Research Lab\\
Department of Computer Science\\
University of South Dakota\\
Vermillion, SD 57069, USA
}

\begin{document}
\maketitle

\begin{abstract}
Spiking neural networks (SNNs) offer an energy-efficient alternative to conventional deep neural networks by exploiting sparse event-driven computation, but their training remains challenging because the non-differentiable spike function requires surrogate gradients whose fixed shape may be suboptimal across layers and training stages. In this work, we introduce SAGE, an uncertainty-modulated surrogate-gradient mechanism for Transformer-based SNNs. SAGE estimates block-level uncertainty from normalized self-attention entropy and uses this signal to adapt the surrogate-gradient slope during training while leaving the inference model unchanged. By modulating only the training-time surrogate parameter, the proposed method preserves the original architecture and deployment cost while improving optimization flexibility. Experiments on CIFAR-10/100 demonstrate that SAGE achieves improved accuracy over fixed-surrogate baselines, with results up to 1-2\% consistent gains across multiple simulation time steps. These results highlight the potential of attention-derived uncertainty as a lightweight training signal for adaptive surrogate-gradient learning in transformer-based SNNs.
\end{abstract}

\section{Introduction}
Spiking Neural Networks (SNNs), recognized as the third generation of neural network models, have emerged as an energy-efficient, event-driven paradigm for deep learning \cite{MAASS19971659, roy2019towards}. By replacing continuous floating-point activations with discrete binary spikes ($0$ or $1$), SNNs running on neuromorphic hardware replace resource-heavy Multiply-Accumulate operations with sparse, addition-only (AC) operations, dramatically reducing computational energy consumption \cite{horowitz20141, kundu2021hire}. Building upon these biologically plausible properties, recent architectural advances have successfully introduced self-attention mechanisms into SNNs. Most notably, the Spikformer architecture \cite{zhou2022spikformer} established Spiking Self-Attention (SSA), which eliminates softmax normalizations to compute sparse, AC attention maps across spike-form Query, Key, and Value tensors. This breakthrough demonstrated that Spiking Vision Transformers (ViTs) can achieve competitive accuracy on large-scale visual benchmarks while maintaining ultra-low hardware energy profiles.

Despite their efficient forward computation, training deep Spiking ViTs remains a fundamental challenge. Since the Heaviside spike-generation function is non-differentiable, Backpropagation Through Time (BPTT) cannot be applied directly. Instead, modern SNN training relies on surrogate gradients that approximate the derivative during backward propagation \cite{lee2016training, neftci2019surrogate, wu2018spatio, zheng2021going}. However, conventional surrogate gradient functions, such as static arctangent, Fast Sigmoid, or piecewise linear derivatives, apply fixed or globally uniform gradient response windows across all spatial sequence dimensions \cite{bellec2018long, zhou2026advancing}. While recent works have explored temporal-wise learnable surrogates to track time-step dynamics \cite{zhou2026advancing}, existing paradigms treat all spatial tokens identically during backpropagation, ignoring the complex, token-level feature dynamics inherent to vision transformer layers.

This spatially uniform gradient approximation creates a critical failure mode in deep Spiking transformers. Visual tokens exhibit vastly different levels of semantic uncertainty: high-confidence tokens (such as distinct object boundaries) require narrow, precise gradient updates, whereas high-uncertainty tokens (such as ambiguous background clutter or complex textures) demand broader, exploratory gradient flow to discover optimal representations. By enforcing a static surrogate window across all tokens, standard backpropagation fails to provide contextual credit assignment. In deep network layers, this token-blind gradient estimation exacerbates severe gradient vanishing and causes attention heads to collapse into ``dead heads", where membrane potentials remain subthreshold and weight updates stall completely \cite{zhou2022spikformer, zhou2026advancing}. Consequently, a major challenge in scaling Spiking ViTs lies in developing an adaptive surrogate mechanism that dynamically tunes gradient flow to local token uncertainty without compromising the deterministic, low-power binary execution of the forward pass.

Motivated by this challenge, we propose SAGE (Surrogate-gradient Adaptation via attention-Guided Entropy), an uncertainty-aware surrogate-gradient framework for Spiking ViTs. Instead of using a fixed or globally learnable surrogate throughout training, SAGE estimates the uncertainty of each transformer block from the dispersion of attention entropy across self-attention heads and adaptively modulates the surrogate-gradient slope during backpropagation. The proposed framework operates exclusively during training, requiring no architectural modifications, additional inference-time parameters, or changes to the forward computation. Extensive experiments on CIFAR-10, CIFAR-100, and ImageNet-200 demonstrate that SAGE consistently improves training effectiveness while preserving the efficiency and deployment characteristics of the original Spikformer. Precisely, our contributions are as follows:
\begin{enumerate}
    \item We introduce the first uncertainty-aware adaptive surrogate-gradient framework for spiking transformers.
    \item We demonstrate that entropy dispersion across attention heads provides a reliable online uncertainty signal for surrogate adaptation and is substantially more discriminative than alternative metrics for characterizing transformer-block uncertainty.
    \item We validate SAGE on competitive datasets, where it consistently outperforms fixed and learnable surrogate baselines by up to 1-2\% Top-1 accuracy while adding only 0.03 ms training overhead per mini-batch.
\end{enumerate}
The remainder of this paper is organized as follows. Section~\ref{sec:motivation} presents the motivation for the proposed framework, establishing the need for uncertainty-aware surrogate adaptation through theoretical insights and supporting literature. Section~\ref{sec:related_work} reviews the existing literature on surrogate-gradient learning and Spiking ViTs. Section~\ref{sec:method} introduces the proposed SAGE framework, detailing the uncertainty estimation strategy and its integration into Spikformer training. Section~\ref{sec:experiments} describes the experimental setup and presents comprehensive evaluations. Section~\ref{sec:discussion} provides a detailed analysis and discussion of the experimental findings, while Section~\ref{sec:conclusion} concludes the paper.

\section{Motivation}
\label{sec:motivation}
Deep SNNs are trained using surrogate gradients, which approximate the derivative of the non-differentiable spike function during backpropagation. Over the past decade, numerous surrogate-gradient formulations have been proposed, differing primarily in their functional forms~\cite{zenke2018superspike}, smoothness characteristics~\cite{shrestha2018slayer}, and gradient scaling strategies~\cite{wu2018spatio}. Despite these differences, existing methods employ a fixed surrogate response for every neuron, treating all spike events identically regardless of the confidence of the underlying representation. This assumption contrasts with modern learning paradigms, where intermediate representations contribute unequally to downstream predictions~\cite{vaswani2017attention}, and uncertainty estimation is widely used to identify ambiguous or unreliable features~\cite{gawlikowski2023survey}. Consequently, applying identical surrogate gradients to neurons encoding both highly certain and highly uncertain information may lead to suboptimal credit assignment during optimization~\cite{neftci2019surrogate}.

These observations naturally raise the following question: \textit{should surrogate gradients remain uniform when the underlying representations exhibit different levels of uncertainty?} We hypothesize that surrogate gradient modulation should instead reflect the uncertainty of the encoded representation. Intuitively, confident representations require less surrogate smoothing, whereas uncertain representations may benefit from broader surrogate support to maintain effective gradient propagation. This intuition is consistent with the broader principles of adaptive computation and uncertainty-aware learning in modern deep networks~\cite{vaswani2017attention,gawlikowski2023survey}. Motivated by this observation, we propose an uncertainty-modulated surrogate gradient that dynamically adapts the backward surrogate response using attention-derived uncertainty while preserving the deterministic forward dynamics of the Spiking ViTs. The proposed formulation is analyzed theoretically in Proposition~\ref{prop:monotonic}, which establishes the monotonic relationship between uncertainty and the surrogate gradient magnitude.

\begin{proposition}
\label{prop:monotonic}
Let $z_c$ denote the centered normalized entropy-dispersion statistic for a transformer block. The adaptive surrogate-gradient slope is defined as
\begin{equation}
\alpha(z_c)=
\begin{cases}
\alpha_0, & |z_c|<\delta,\\
\alpha_0+\beta\tanh(z_c), & |z_c|\ge\delta,
\end{cases}
\end{equation}
where $\alpha_0>0$, $\beta>0$, and $\delta>0$. Then, outside the dead-zone region, $\alpha(z_c)$ is a monotonically increasing function of $z_c$.
\end{proposition}
\begin{proof}
For $|z_c|\ge\delta$,
\begin{equation}
\frac{d\alpha}{dz_c}=\beta\left(1-\tanh^2(z_c)\right).
\end{equation}
Since $\beta>0$ and\(1-\tanh^2(z_c)>0\)
for all finite $z_c$,
\begin{equation}
\frac{d\alpha}{dz_c}>0.
\end{equation}
Therefore, $\alpha$ is strictly monotonic with respect to the centered uncertainty statistic outside the dead-zone region. Within the dead zone, $\alpha=\alpha_0$ is constant, intentionally suppressing small uncertainty fluctuations.
\end{proof}

\section{Related Works}
\label{sec:related_work}
This section reviews the three research directions most closely related to the proposed SAGE framework. We first summarize advances in direct training of Spiking Neural Networks, where surrogate gradients have become the standard approach for optimizing deep SNNs. We then discuss the evolution of Spiking ViTs, highlighting recent architectural developments built upon self-attention mechanisms. Finally, we review existing surrogate-gradient optimization strategies and their limitations, thereby motivating the need for an uncertainty-guided adaptive surrogate framework for spiking transformers.
\paragraph{Spiking Neural Networks \& Direct Training} SNNs represent a biologically inspired, event-driven paradigm that achieves significant energy efficiency by transmitting discrete binary spikes rather than continuous floating-point activations \cite{ghosh2009spiking, tavanaei2019deep}. Early efforts to scale SNNs relied on Artificial Neural Network (ANN)-to-SNN conversion, which typically incurred excessive latency and high inference time steps \cite{han2020rmp}. To overcome these latency bottlenecks, direct training via Spatio-Temporal Backpropagation (STBP) and BPTT emerged as the dominant framework for optimizing deep architectures \cite{lee2016training, wu2018spatio}. Direct training algorithms substitute the non-differentiable Heaviside step function with smooth Surrogate Gradients during backward propagation, enabling competitive low-latency execution \cite{neftci2019surrogate}. Recent advances have expanded direct SNN training across diverse domains, including dynamic spiking graph networks \cite{yin2024dynamic}, asynchronous event-based vision processors \cite{zeng2025leveraging}, and attention-guided spiking architectures \cite{yao2023attention}. Nevertheless, a fundamental limitation persists: direct training algorithms rely heavily on static or globally uniform surrogate derivative functions \cite{neftci2019surrogate, zhou2026advancing}, which fail to accommodate localized spatial variations in complex feature representations.
\paragraph{Spiking Vision Transformers}
To combine the representation capability of self-attention \cite{vaswani2017attention} with the low-power characteristics of neuromorphic processing, recent studies have adapted ViTs to spiking domain operations. The seminal Spikformer architecture \cite{zhou2022spikformer} introduced SSA, which eliminates softmax normalizations to compute sparse, AC attention maps across spike-form Query, Key, and Value tensors. Building on this foundation, subsequent work introduced multiscale spiking ViTs \cite{yu2024spikingvit}, hybrid event-based detection models \cite{xu2025hybrid}, biological visual mechanisms like saccadic attention \cite{wang2025spiking}, and structural bridges between ResNets and ViTs \cite{shi2024spikingresformer}. Additionally, gating mechanisms \cite{Nair_2026_CVPR} and dynamic time-step allocation frameworks \cite{datta2025dynamic} have been explored to enhance information flow control and reduce latency. However, while these architectures optimize forward-pass execution and structural feature routing, their backward optimization remains constrained by static surrogate functions \cite{zhou2022spikformer}, leaving deep spiking transformer layers susceptible to vanishing gradients and dead attention heads.
\paragraph{Surrogate Gradient Optimization}
Overcoming the non-differentiability of spiking activations has motivated extensive research into surrogate gradient formulation. Conventional direct training paradigms employ fixed analytical derivatives, such as Fast Sigmoid, arctangent, or piecewise linear functions \cite{bellec2018long}. To mitigate gradient vanishing and mismatch, researchers have developed adaptive and learnable surrogate functions \cite{perez2021sparse, che2022differentiable, guo2024take}, including adaptive smoothing gradient learning \cite{wang2023adaptive}, cross-layer threshold adaptations \cite{ai2025cross}, and dual-stage threshold-gradient optimization \cite{jiang2026ds}. Recent extensions have further explored lightweight adaptive surrogates \cite{hou2026adaptive}, membrane-potential-driven gradient scaling \cite{jiang2025adaptive, sun2026optimization}, sparse low-activity firing constraints \cite{stanojevic2024high}, and adaptive surrogates for sequential reinforcement learning \cite{van2026adaptive}. Although these approaches dynamically modulate gradient shapes, they primarily adapt across global temporal dimensions or layer-wide statistics \cite{jiang2025adaptive, zhou2026advancing}. Inspired by uncertainty-guided learning \cite{qiao2021uncertainty, lai2024uncertainty}, confidence-aware optimization \cite{moon2020confidence, lv2026confidence}, and causal entropy principles \cite{branchini2023causal, cheng2026group}, there remains an unmapped frontier: formulating an uncertainty-modulated, token-level surrogate gradient that dynamically tunes its derivative window based on spatial attention entropy during backpropagation.
\begin{figure*}[!tbp]
    \centering
    \includegraphics[width=\textwidth]{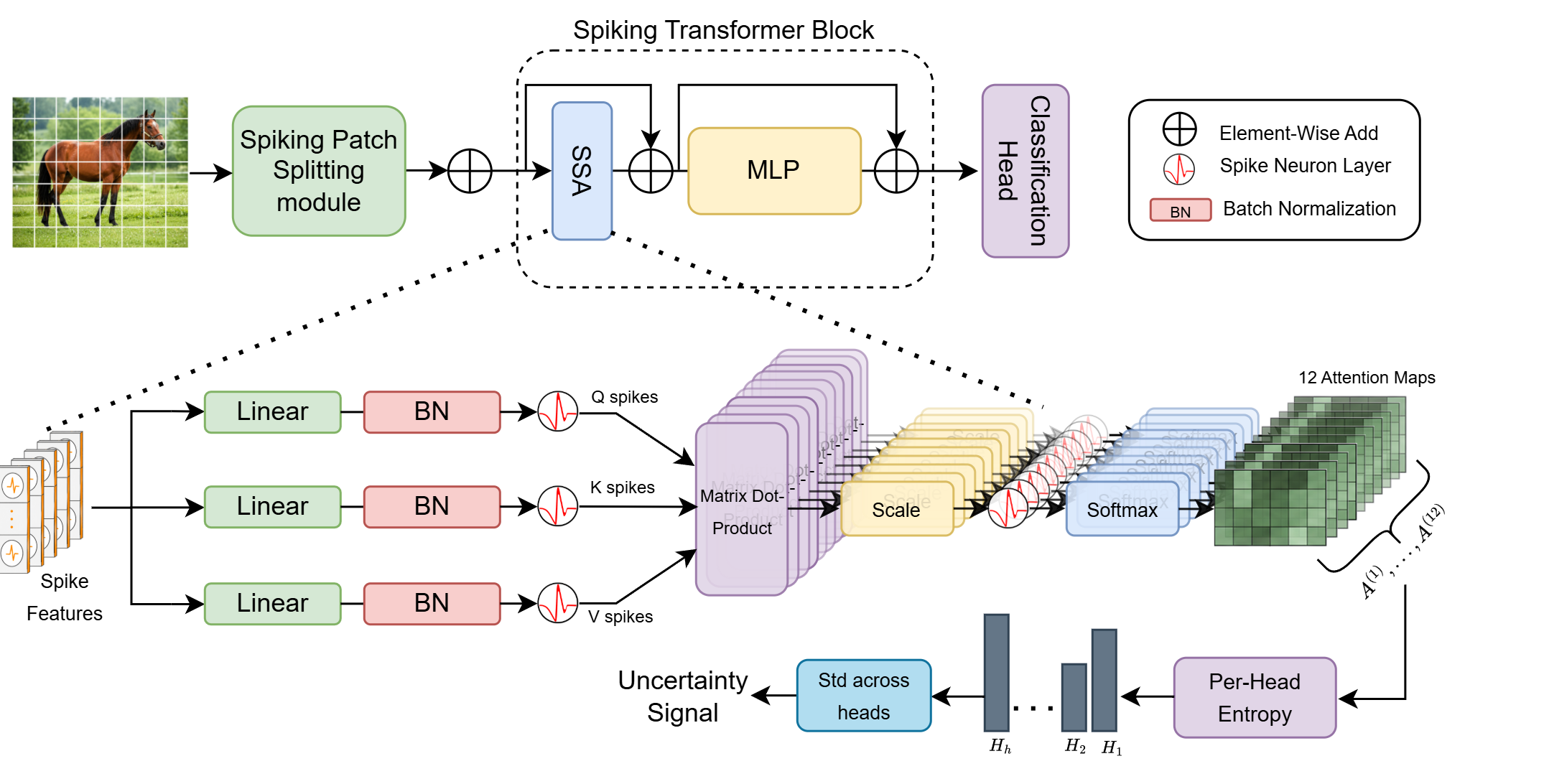}
    \caption{
    Overview of the proposed \textbf{SAGE} (\textbf{S}urrogate-gradient \textbf{A}daptation via attention-\textbf{G}uided \textbf{E}ntropy) framework. 
    \textbf{Top:} Standard Spikformer architecture consisting of a spiking patch splitting module, SSA, Multilayer Perceptron(MLP) blocks, and a classification head. 
    \textbf{Bottom:} Detailed view of the SSA module. Multi-head attention (MHA) maps are generated from the spike-based query ($Q$), key ($K$), and value ($V$) representations. Per-head attention entropy is computed from each attention map, and the standard deviation across attention heads is used as an uncertainty signal to guide adaptive surrogate-gradient modulation during training. The forward inference path remains unchanged, while the uncertainty-guided controller is active only during backpropagation.
    }
    \label{fig:architecture}
\end{figure*}

\section{Method}
\label{sec:method}
\subsection{Overview of SAGE}
Figure~\ref{fig:architecture} presents an overview of the proposed \textbf{SAGE} (\textbf{S}urrogate-gradient \textbf{A}daptation via attention-\textbf{G}uided \textbf{E}ntropy) framework for adaptive surrogate-gradient optimization in spiking transformers. Built upon the standard Spikformer architecture~\cite{zhou2022spikformer}, SAGE preserves the original forward inference pipeline and introduces no modifications to the network architecture or inference procedure. Instead, the proposed framework operates exclusively during training by adaptively modulating the surrogate gradient used for backpropagation. The central idea is that the multi-head SSA module naturally produces attention distributions whose variability reflects the model's uncertainty during optimization. Rather than applying a fixed surrogate gradient uniformly across all transformer blocks throughout training, SAGE derives an uncertainty signal from the attention maps of each block and uses it to dynamically adjust the surrogate gradient parameter, enabling different blocks to receive optimization behavior that is consistent with their current uncertainty.

Importantly, SAGE is a training-time optimization framework and does not alter the inference process. During inference, both the uncertainty estimation module and the adaptive controller are removed, leaving the original Spikformer architecture unchanged. Consequently, SAGE preserves the inference graph, computational complexity, and latency of the baseline model while remaining readily applicable to existing spiking transformer architectures without introducing additional inference overhead.

\subsection{Attention Uncertainty Estimation}
\label{subsec:attention_uncertainity}
To derive a training signal that reflects the confidence of the network, SAGE exploits the attention distributions naturally produced by the SSA module, as illustrated in Figure~\ref{fig:architecture}. Given the query and key spike representations, SSA module computes the raw attention scores as
\begin{equation}
\mathbf{S}=s\,\mathbf{Q}\mathbf{K}^{\top},
\label{eq:attention_score}
\end{equation}
where $s=0.125$ is the fixed scaling factor used in the Spikformer~\cite{zhou2022spikformer} implementation. These raw attention scores are used directly in the standard SSA forward computation and are not normalized by a softmax operation. For uncertainty estimation only, SAGE constructs an auxiliary probability distribution from the detached attention scores. Specifically, the scores are temperature-scaled and normalized as
\begin{equation}
\mathbf{A}=\mathrm{Softmax}\left(
\frac{\mathrm{detach}(\mathbf{S})}{T_{\mathrm{ent}}}\right),
\label{eq:attention_map}
\end{equation}
where $\mathrm{detach}(\cdot)$ indicates that the attention scores are excluded from gradient computation, ensuring that the uncertainty estimation influences only the surrogate-gradient controller and does not modify the forward attention computation. The entropy temperature is fixed to $T_{\mathrm{ent}}=0.25$, selected based on a one-factor-at-a-time (OFAT) sensitivity analysis over multiple temperature values (see Appendix figure~\ref{fig:temperature_analysis}). The resulting auxiliary distributions $\mathbf{A}^{(1)},\mathbf{A}^{(2)},\ldots,\mathbf{A}^{(H)}$
correspond to the $H$ attention heads and are used exclusively to compute the uncertainty signal for SAGE. For each attention head, we quantify the uncertainty of its attention distribution using the normalized Shannon entropy~\cite{shannon1948mathematical},
\begin{equation}
E_i=-\frac{1}{\log N}\sum_{j=1}^{N}A^{(i)}_j\log A^{(i)}_j,
\label{eq:entropy}
\end{equation}
where $N$ is the number of attention elements within a head and the normalization confines the entropy to the range $[0,1]$. While the mean entropy reflects the overall uncertainty of the attention mechanism, our discussion (Section~\ref{sec:discussion}) shows that the variation of entropy across attention heads is considerably more informative during training. Therefore, SAGE estimates the uncertainty of each transformer block using the dispersion of the per-head entropies,
\begin{equation}
D=\mathrm{Std}\left(E_1,E_2,\ldots,E_H\right),
\label{eq:dispersion}
\end{equation}
where $\mathrm{Std}(\cdot)$ denotes the standard deviation across the $H$ attention heads. A larger value of $D$ indicates greater disagreement among attention heads, suggesting higher uncertainty in the current representation, whereas a smaller value implies more consistent attention patterns. This dispersion statistic serves as the uncertainty signal that drives the adaptive surrogate-gradient controller described in the following subsection.
\begin{table}[t]
\centering
\caption{Comparison of state-of-the-art SNN and Spiking Transformer methods on CIFAR-10 and CIFAR-100.}
\label{tab:cifar_comparison}
\renewcommand{\arraystretch}{1.10}
\setlength{\tabcolsep}{4pt}
\small
\resizebox{\textwidth}{!}{%
\begin{tabular}{l l c c c c}
\toprule
\textbf{Method} & \textbf{Architecture} & \textbf{Params (M)} & \textbf{Timesteps} & \textbf{CIFAR-10 (\%)} & \textbf{CIFAR-100 (\%)} \\
\midrule
Hybrid Training~\cite{rathi2020enabling} & VGG-11 & 9.27 & 125 & 92.22 & 67.87 \\
Diet-SNN~\cite{rathi2020diet} & ResNet-20 & 0.27 & 10/5 & 92.54 & 64.07 \\
STBP-tdBN~\cite{zheng2021going} & ResNet-19 & 12.63 & 4 & 92.92 & 70.86 \\
TET~\cite{deng2022temporal} & ResNet-19 & 12.63 & 4 & 94.44 & 74.47 \\
Spikformer~\cite{zhou2022spikformer} & Spikformer-4-384 & 9.32 & 4 & 95.51 & 78.21 \\
RMP-SNN~\cite{han2020rmp} & VGG-16 & 138.4 & 2048 & 93.63 & 70.93 \\
QCFS~\cite{bu2023optimal} & ResNet-20 & 0.27 & 32 & 93.25 & 69.82 \\
\midrule
\textbf{SAGE (Ours)} & \textbf{Spikformer-4-384} & \textbf{9.32} & \textbf{4} & \textbf{96.32} & \textbf{78.69} \\
\bottomrule
\end{tabular}%
}
\end{table}

\subsection{Uncertainty-Guided Surrogate Modulation}
\label{subsec:adaptive_controller}
Conventional SNNs employ a fixed surrogate-gradient parameter throughout training, resulting in identical gradient characteristics regardless of the optimization state or the confidence of intermediate representations. In contrast, SAGE adapts the surrogate gradient according to the uncertainty signal extracted from the attention mechanism. Let $D_t^{(l)}$ denote the entropy-dispersion estimate of transformer block $l$ at training step $t$, obtained from Eq.~(\ref{eq:dispersion}). Rather than using a fixed surrogate parameter (e.g., $\alpha=4$), it computes a block-specific surrogate parameter that evolves during optimization.
To obtain a stable control signal, the instantaneous dispersion is first smoothed using an exponential moving average (EMA),
\begin{equation}
\bar{D}_t^{(l)}=\beta\bar{D}_{t-1}^{(l)}+(1-\beta)D_t^{(l)},
\label{eq:ema}
\end{equation}
where $\beta$ denotes the EMA decay factor. The running mean $\mu_t^{(l)}$ and standard deviation $\sigma_t^{(l)}$ of the smoothed dispersion are then maintained throughout training. Using these statistics, the current uncertainty is normalized as
\begin{equation}
z_t^{(l)}=\frac{\bar{D}_t^{(l)}-\mu_t^{(l)}}{\sigma_t^{(l)}+\varepsilon},
\label{eq:zscore}
\end{equation}
where $\varepsilon$ is a small constant for numerical stability. The normalized uncertainty is further centered across the $L$ transformer blocks to obtain a relative block-wise uncertainty score,
\begin{equation}
\tilde{z}_t^{(l)}=z_t^{(l)}-\frac{1}{L}\sum_{r=1}^{L}z_t^{(r)},
\label{eq:centered_z}
\end{equation}
where $\tilde{z}_t^{(l)}$ denotes the centered normalized entropy-dispersion statistic of block $l$ at training step $t$. To suppress minor fluctuations around the block-wise mean, SAGE applies a dead-zone controller before modulating the surrogate-gradient slope,
\begin{equation}
\hat{\alpha}_t^{(l)}=\begin{cases}
4.0, & |\tilde{z}_t^{(l)}| < 0.25, \\[4pt]
4.0 + 0.5\tanh\!\left(\tilde{z}_t^{(l)}\right),& \text{otherwise},
\end{cases}
\label{eq:adaptive_alpha_raw}
\end{equation}
followed by
\begin{equation}
\alpha_t^{(l)}=\mathrm{clip}\left(\hat{\alpha}_t^{(l)},\,3,\,5\right).
\label{eq:alpha}
\end{equation}
The centering operation enables SAGE to adapt the surrogate gradient relative to the uncertainty distribution across transformer blocks, while the dead zone prevents small variations from triggering unnecessary updates. The $\tanh(\cdot)$ mapping provides a smooth bounded modulation around the baseline value $\alpha=4$, and the final clipping operation constrains the surrogate slope to the interval $[3,5]$. Consequently, transformer blocks exhibiting higher uncertainty receive a different surrogate gradient than more confident blocks, enabling the optimization process to adapt to the evolving attention dynamics while maintaining stable training. During the initial warm-up stage, the surrogate parameter is fixed at $\alpha=4$ while the controller accumulates reliable running statistics before adaptive modulation begins.

\subsection{Integration into Spikformer Training}
SAGE is designed as a lightweight training-time optimization framework that can be integrated into existing Spikformer implementations with minimal modification. During the forward pass, the network architecture, feature extraction process, and inference computation remain identical to the original Spikformer. The proposed framework simply accesses the MHA maps generated by each SSA module and computes the uncertainty estimate described in section~\ref{subsec:attention_uncertainity}. The resulting uncertainty signal is then passed through the adaptive controller in Section~\ref{subsec:adaptive_controller} to obtain a block-specific surrogate parameter $\alpha$, which is used only during backpropagation. Consequently, SAGE introduces no architectural changes, additional learnable parameters, or inference-time overhead.
\begin{figure*}[tbp]
    \centering

    \begin{subfigure}[t]{0.32\textwidth}
        \centering
        \includegraphics[width=\linewidth]{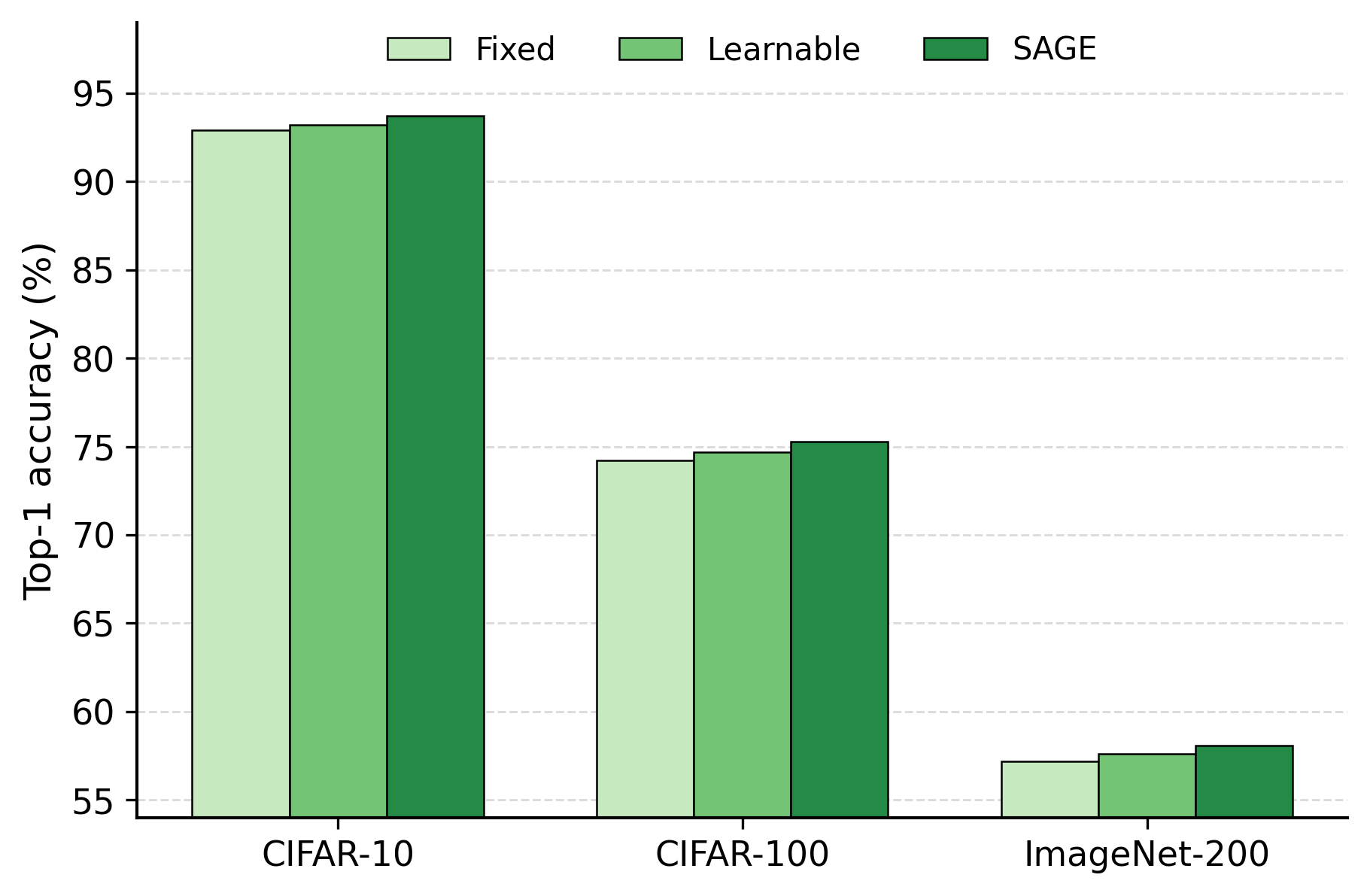}
        \caption{$T=1$}
    \end{subfigure}
    \hfill
    \begin{subfigure}[t]{0.32\textwidth}
        \centering
        \includegraphics[width=\linewidth]{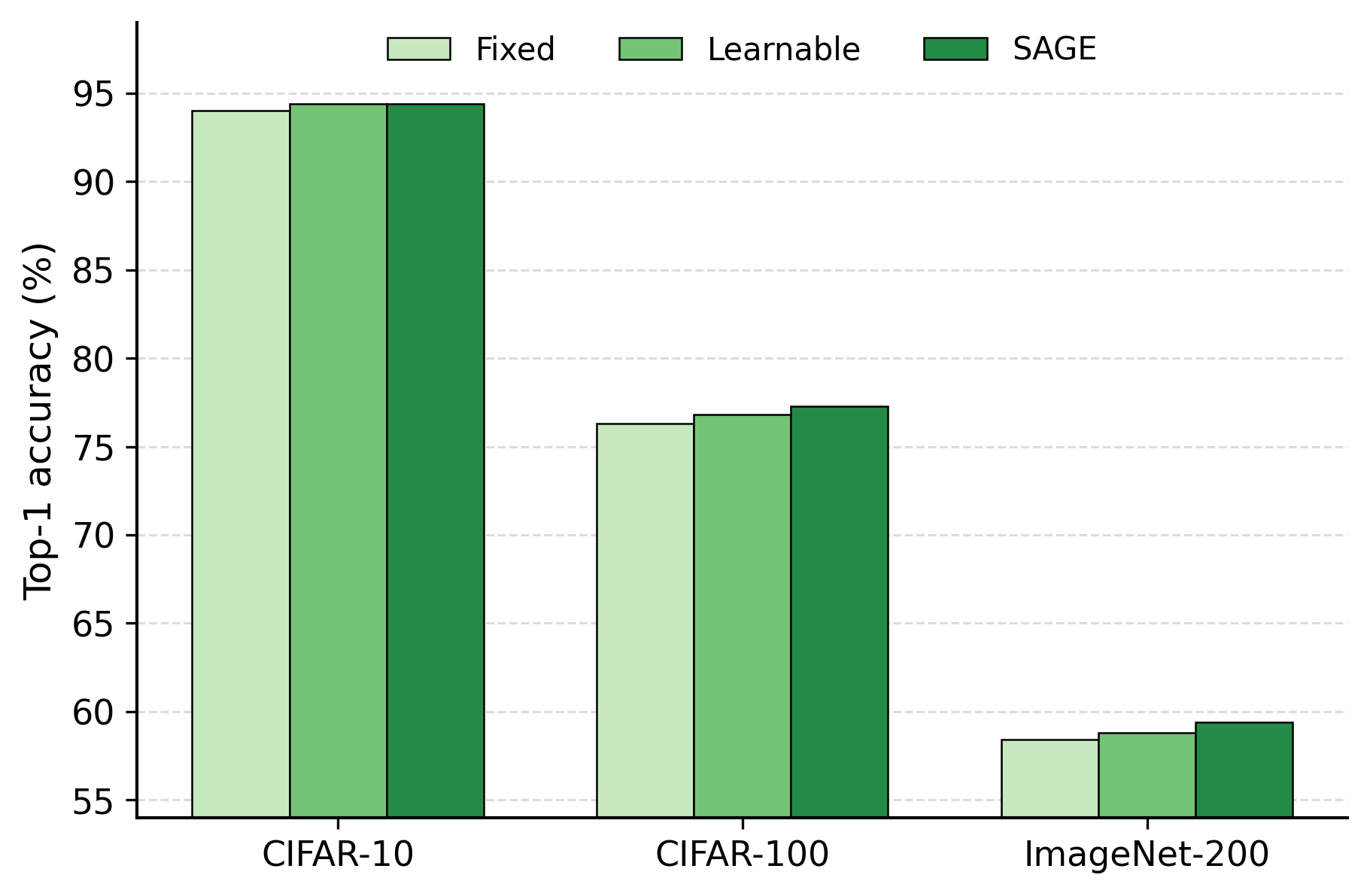}
        \caption{$T=2$}
    \end{subfigure}
    \hfill
    \begin{subfigure}[t]{0.32\textwidth}
        \centering
        \includegraphics[width=\linewidth]{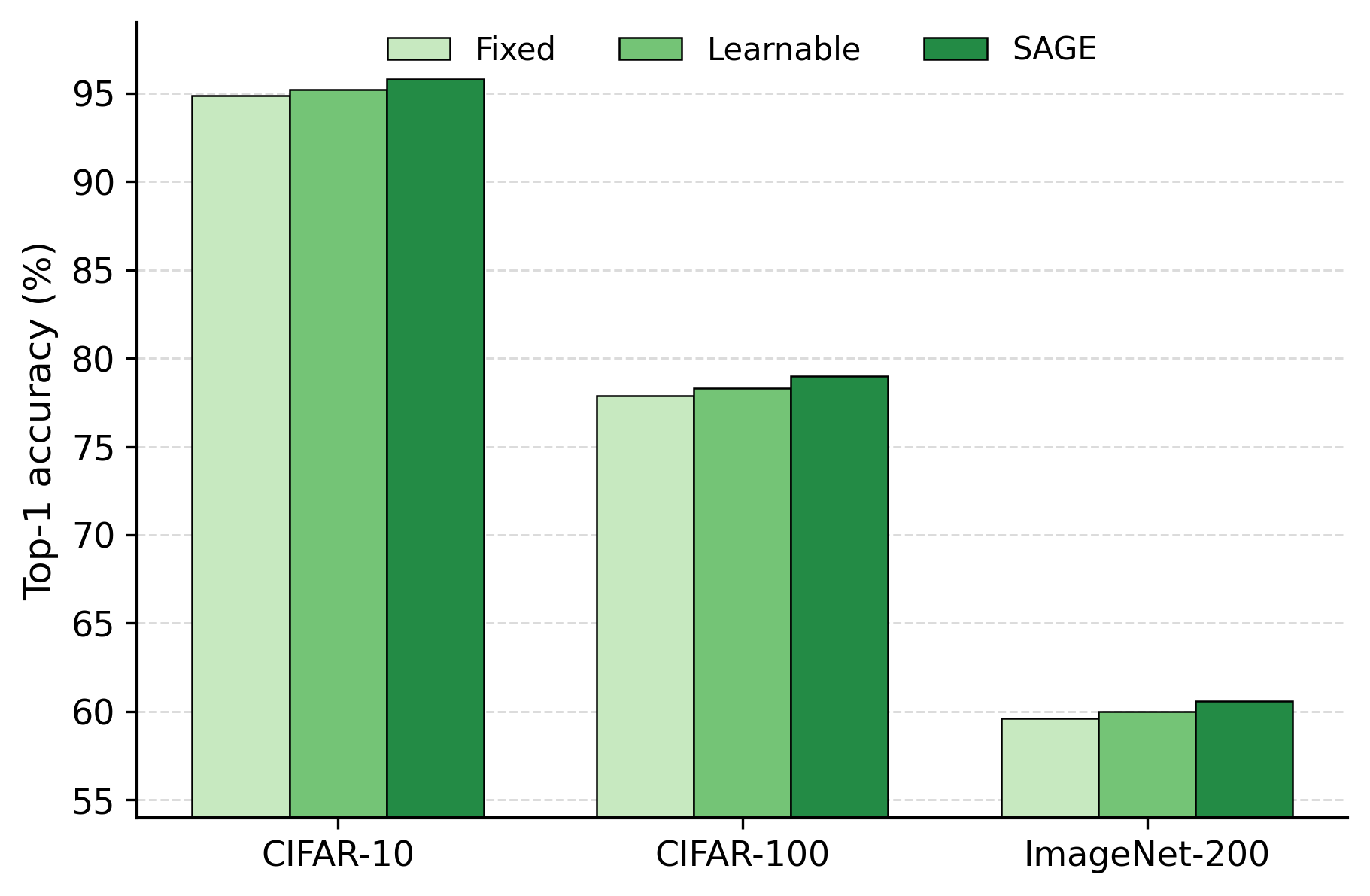}
        \caption{$T=4$}
    \end{subfigure}

    \vspace{0.5em}

    \begin{subfigure}[t]{0.32\textwidth}
        \centering
        \includegraphics[width=\linewidth]{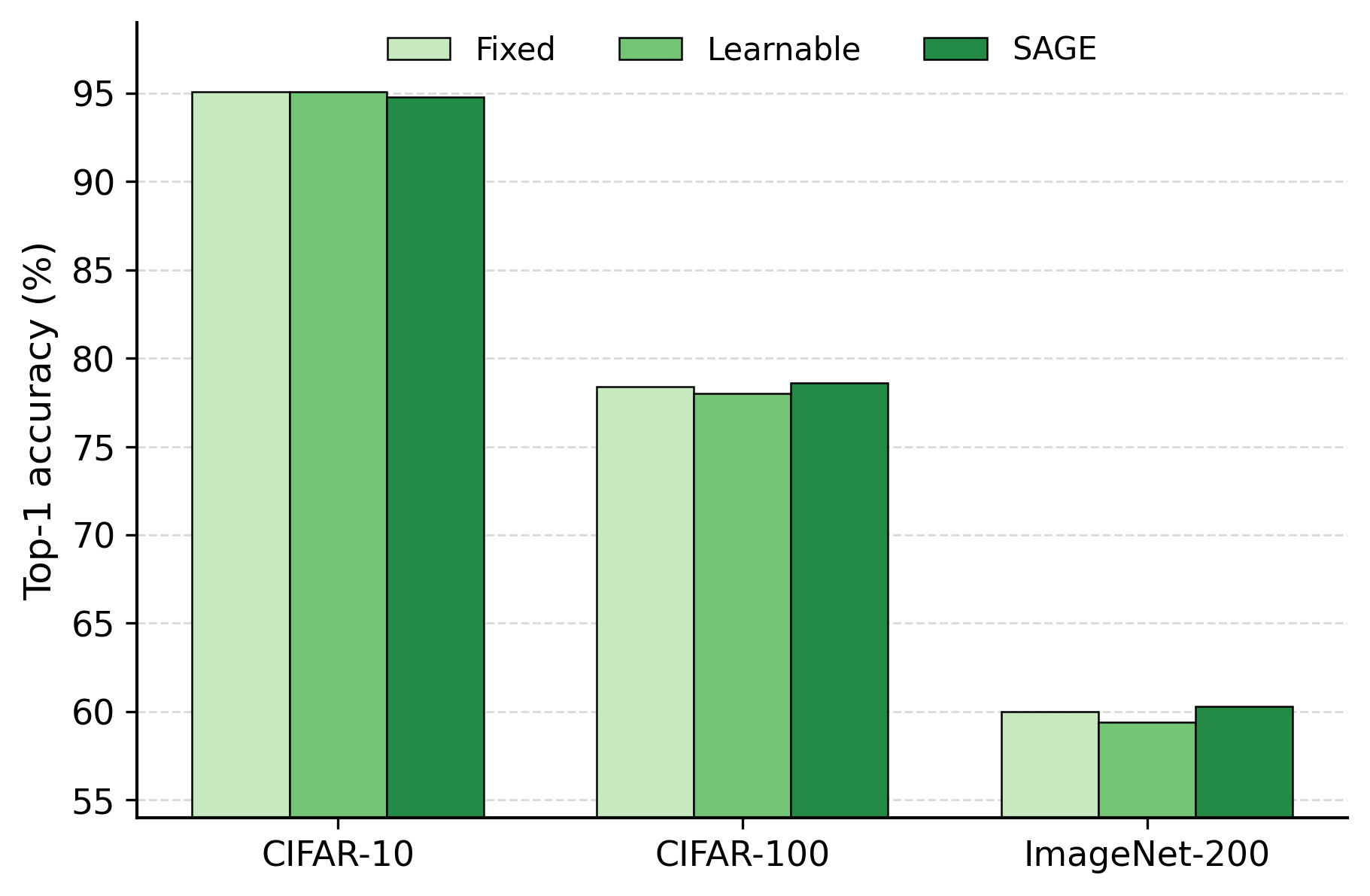}
        \caption{$T=8$}
    \end{subfigure}
    \hspace{0.04\textwidth}
    \begin{subfigure}[t]{0.32\textwidth}
        \centering
        \includegraphics[width=\linewidth]{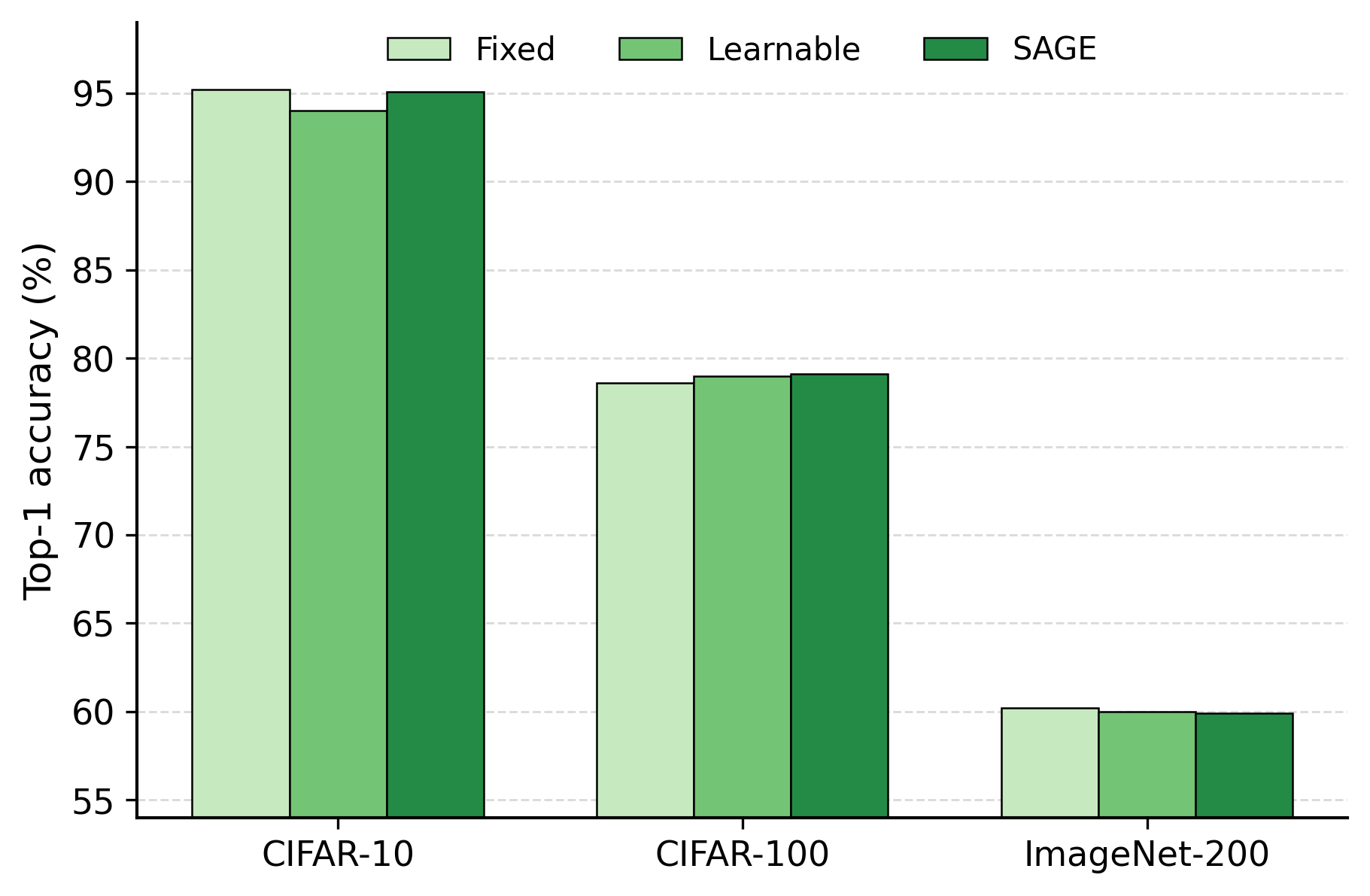}
        \caption{$T=12$}
    \end{subfigure}

    \caption{Top-1 accuracy of fixed, learnable, and SAGE surrogate gradients across datasets and simulation timesteps.}
    \label{fig:surrogate_timestep_comparison}
\end{figure*}

The additional computation introduced by SAGE is negligible compared with the cost of self-attention. For an input sequence of length $N$ and $H$ attention heads, the dominant complexity of the attention mechanism remains the scaled dot-product computation with complexity $\mathcal{O}(N^2)$. In comparison, the proposed uncertainty estimation computes the entropy of each attention head in $\mathcal{O}(HN)$, followed by a standard deviation across heads with complexity $\mathcal{O}(H)$. The subsequent exponential moving average, running statistics, and adaptive controller involve only constant-time operations per transformer block. Since $\mathcal{O}(HN) \ll \mathcal{O}(N^2)$ for practical transformer configurations, the overall computational complexity of Spikformer remains unchanged. Experimental profiling further confirms that the adaptive controller incurs only approximately $0.03$ ms of additional computation per mini-batch during training, while introducing zero overhead during inference.
\section{Experimental Results}
\label{sec:experiments}
We conduct comprehensive experiments to evaluate the effectiveness of the proposed SAGE framework on image classification benchmarks. SAGE is evaluated using the Spikformer backbone and compared with both the original Spikformer baseline and recent state-of-the-art SNNs. In addition, extensive ablation studies are performed to analyze the contribution of each component of the proposed adaptive surrogate-gradient strategy.
\subsection{Experimental Setup}
\subsubsection{Datasets}
Experiments are conducted on three widely used image classification benchmarks: CIFAR-10, CIFAR-100~\cite{krizhevsky2009learning}, and ImageNet-200~\cite{deng2009imagenet}. CIFAR-10 and CIFAR-100 each contain 60,000 RGB images of resolution $32\times32$, divided into 50,000 training and 10,000 testing samples. CIFAR-10 consists of 10 object categories, whereas CIFAR-100 contains 100 fine-grained classes, providing a more challenging classification benchmark due to increased inter-class similarity. ImageNet-200 is a commonly used subset of the ImageNet benchmark containing 200 object categories with approximately 128k training images and 10k validation images, offering substantially greater visual diversity while maintaining manageable computational cost.

\subsubsection{Implementation Details}
\label{subsec:hyperparms}
SAGE is implemented on top of the official Spikformer~\cite{zhou2022spikformer} framework without modifying the network architecture or inference pipeline. Unless otherwise specified, experiments employ the Spikformer-4-384 backbone comprising four transformer blocks, an embedding dimension of 384, 12 attention heads, a patch size of 4, an MLP expansion ratio of 4, and four simulation time steps. Models are trained from scratch following the official Spikformer training protocol, including RandAugment, MixUp, Random Erasing, label smoothing, cosine learning-rate scheduling with warmup, and the AdamW optimizer. All experiments are implemented in PyTorch using the SpikingJelly framework and are conducted on NVIDIA Tesla V100 GPUs.

During training, SAGE computes a block-level uncertainty signal from the normalized attention entropy obtained from temperature-scaled attention maps ($T_{\mathrm{ent}}=0.25$). The entropy dispersion across attention heads is smoothed using an exponential moving average, normalized using running statistics, and mapped to an adaptive surrogate-gradient slope through the controller described in Section~\ref{subsec:adaptive_controller}. The surrogate parameter is initialized with the standard SpikingJelly value ($\alpha=4$) during the warmup stage before adaptive modulation is activated. Importantly, SAGE modifies only the surrogate-gradient computation during training, while the forward computation graph, network parameters, and inference procedure remain identical to the original Spikformer.
\begin{figure*}[tbp]
    \centering

    \includegraphics[width=0.48\textwidth]{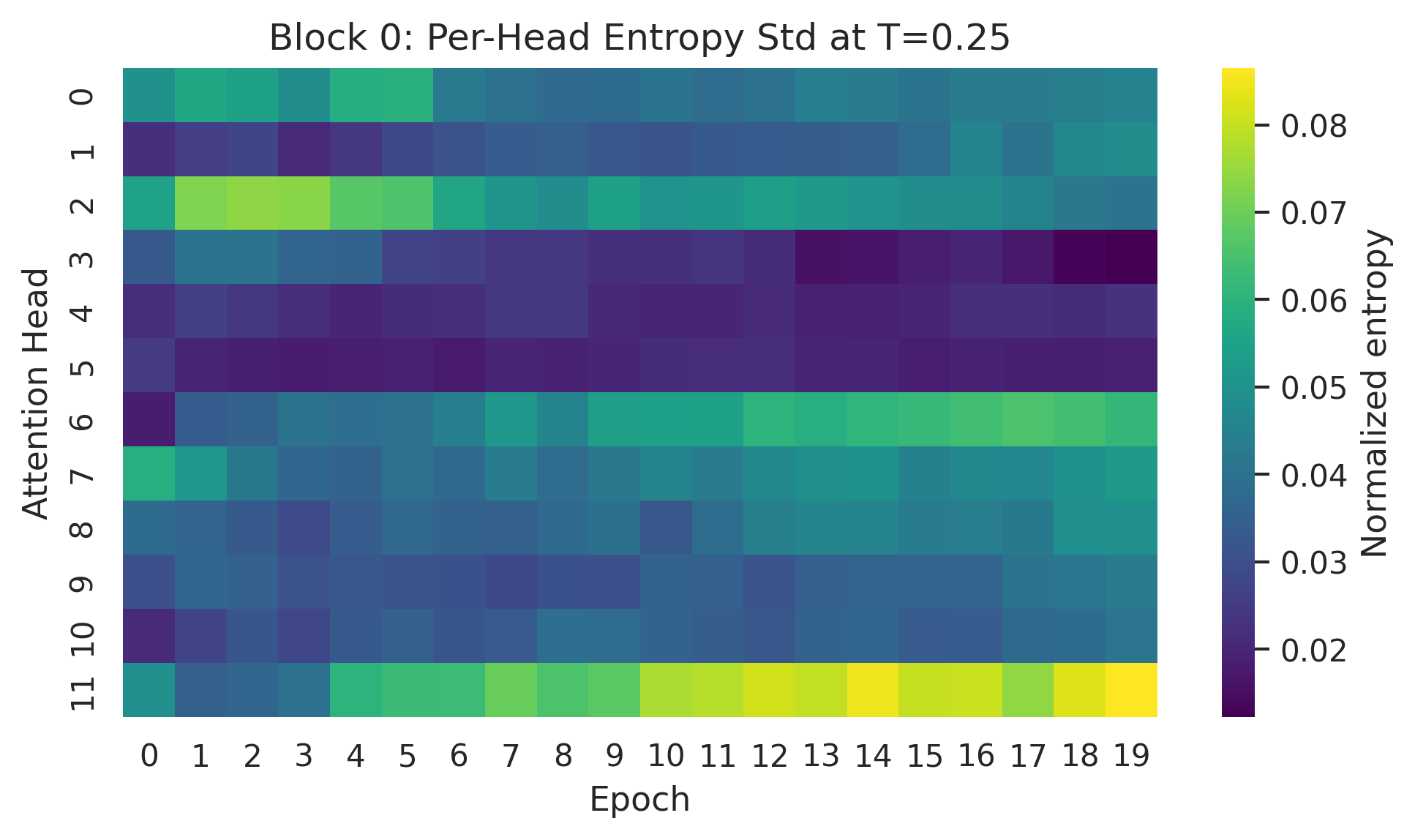}
    \hfill
    \includegraphics[width=0.48\textwidth]{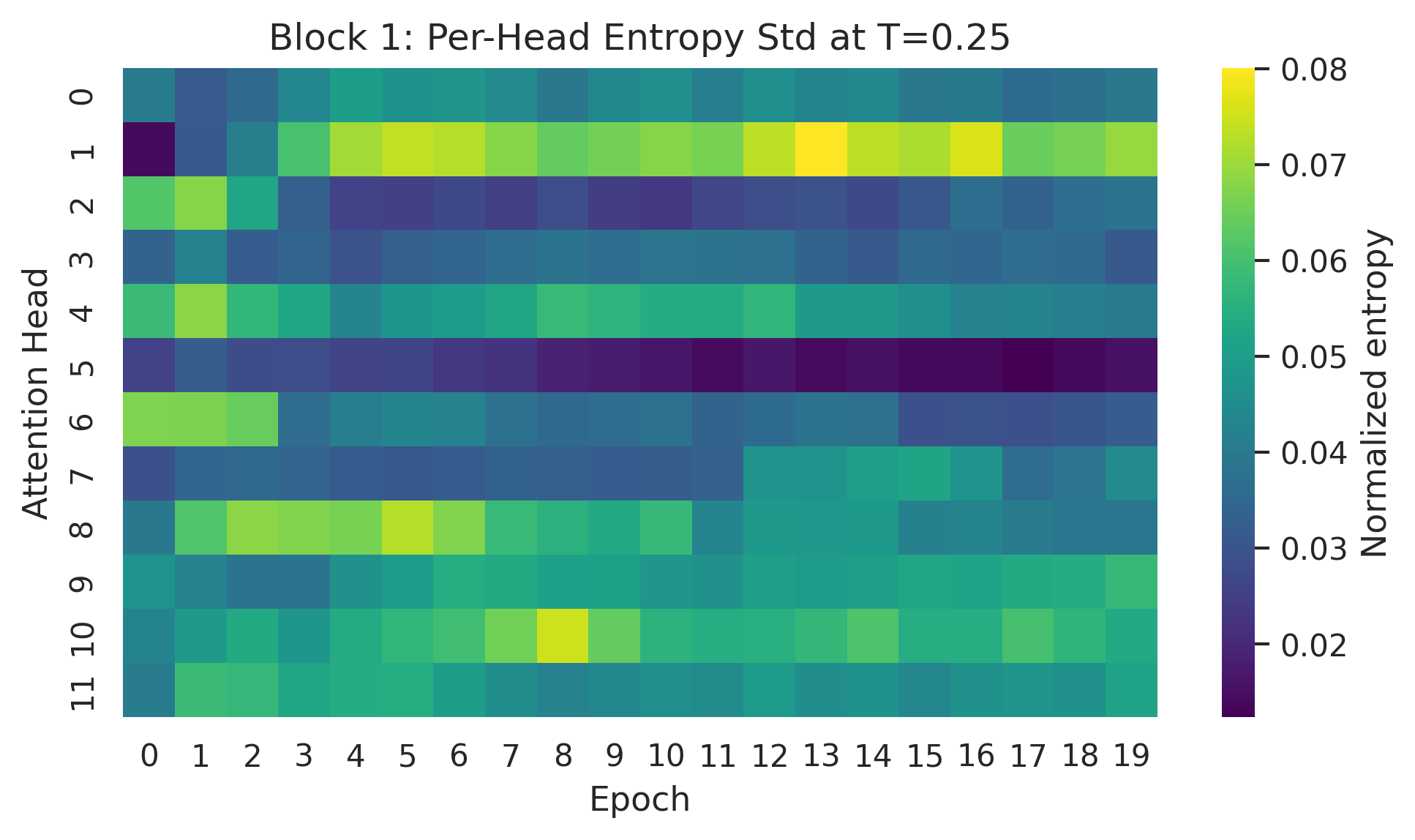}

    \vspace{0.5em}

    \includegraphics[width=0.48\textwidth]{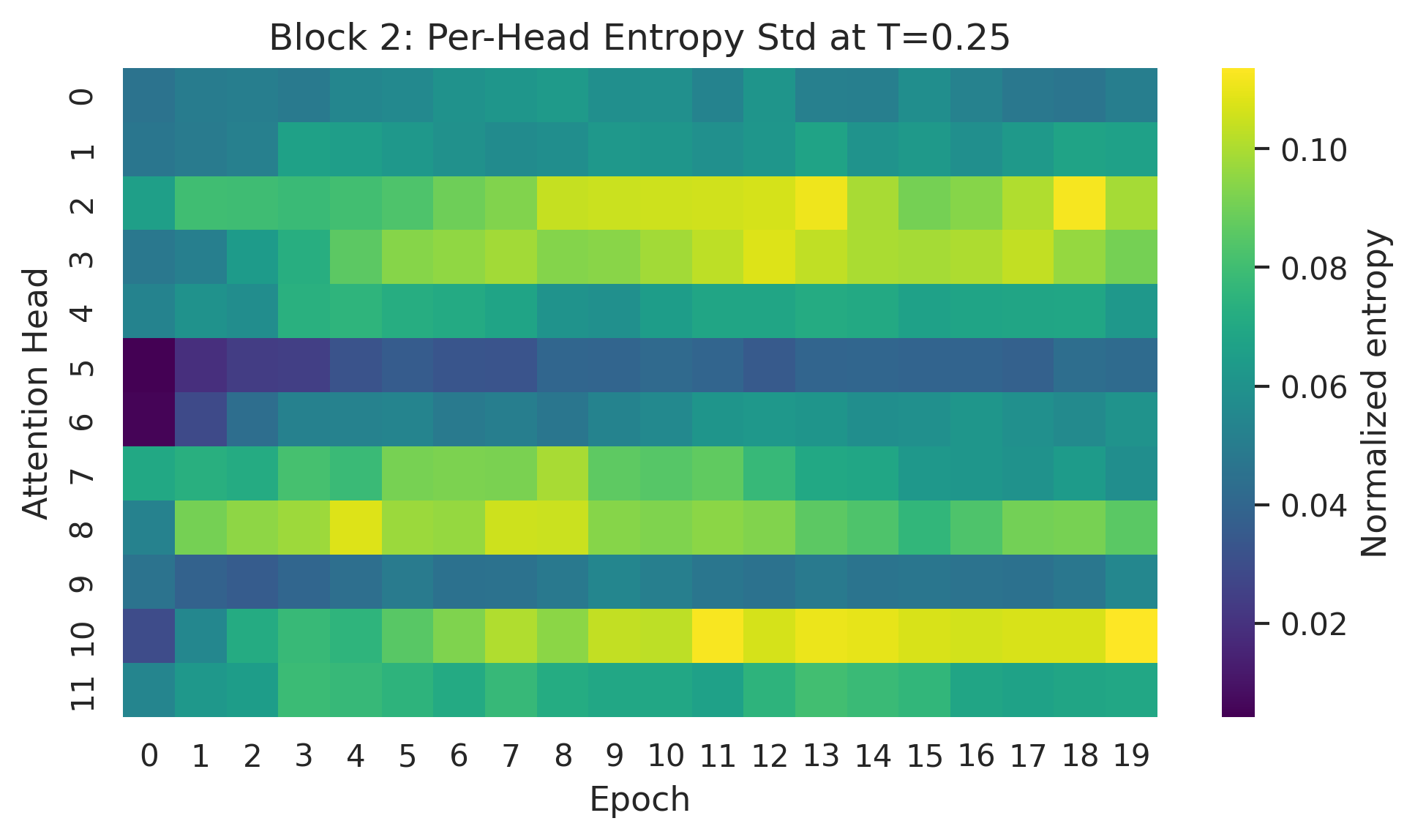}
    \hfill
    \includegraphics[width=0.48\textwidth]{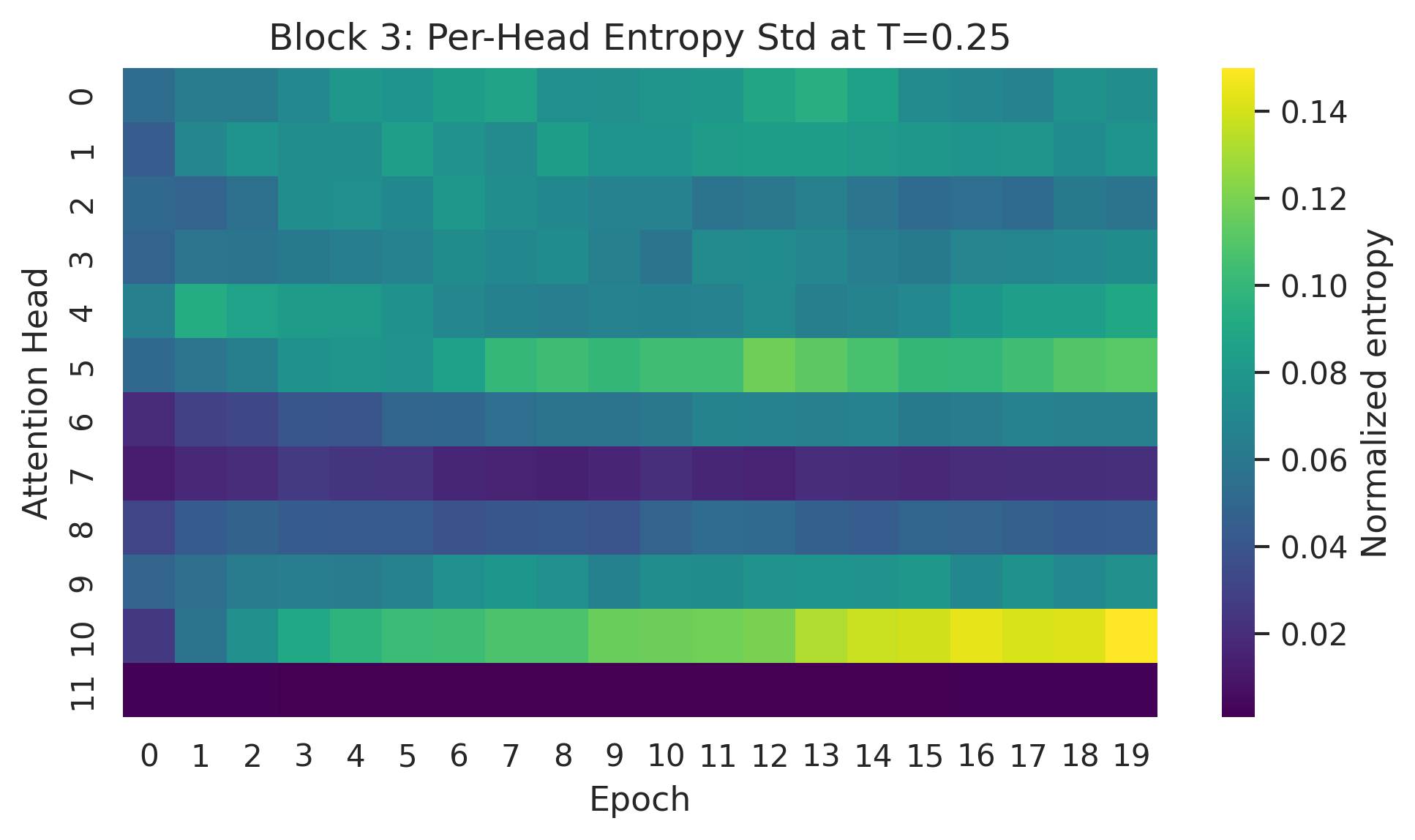}

    \caption{Per-head attention entropy variability across the four SSA blocks during training. Each heatmap visualizes the standard deviation of the normalized attention entropy for each attention head over training epochs, computed from temperature-scaled attention maps ($T=0.25$). Rows correspond to attention heads, columns represent training epochs.}
    \label{fig:per_head_entropy_variability}
\end{figure*}
\begin{table}[tb]
\centering
\caption{Comparison on the ImageNet-200 benchmark.}
\label{tab:imagenet200}
\renewcommand{\arraystretch}{1.15}
\begin{tabular}{l c}
\toprule
\textbf{Method} & \textbf{Top-1 (\%)} \\
\midrule
Hybrid Training~\cite{rathi2020enabling} & 61.48 \\
STBP-tdBN~\cite{zheng2021going} & 63.72 \\
TET~\cite{deng2022temporal} & 63.72 \\
QCFS~\cite{bu2023optimal} & 53.54 \\
Spikformer~\cite{zhou2022spikformer} & 70.24 \\
\midrule
SAGE (Ours) & 60.60 \\
\bottomrule
\end{tabular}
\end{table}

\subsection{Our Results}
Table~\ref{tab:cifar_comparison} and Table~\ref{tab:imagenet200} summarizes the performance of the proposed SAGE framework on CIFAR-10, CIFAR-100, and ImageNet-200, together with representative state-of-the-art SNN and Spiking Transformer methods. Following common practice, the reported results for Hybrid Training~\cite{rathi2020enabling}, DIET-SNN~\cite{rathi2020diet}, STBP-tdBN~\cite{zheng2021going}, TET~\cite{deng2022temporal}, RMP-SNN~\cite{han2020rmp}, QCFS~\cite{bu2023optimal}, and Spikformer~\cite{zhou2022spikformer} are reproduced from the original publications and the recent benchmark study in~\cite{zhou2026advancing}. No re-implementation or retraining of these methods was performed. Our experiments were conducted only on the Spikformer baseline and the proposed SAGE framework using the same training protocol and hyperparameter settings described in Section~\ref{subsec:hyperparms}.

To evaluate the effect of adaptive surrogate-gradient modulation, we additionally compare three surrogate-gradient configurations: Fixed, Learnable, and the proposed SAGE. The Fixed configuration employs the conventional sigmoid surrogate with a constant slope $\alpha=4$ throughout training. The Learnable configuration replaces the fixed slope with a single trainable surrogate-gradient parameter optimized jointly with the network parameters. In contrast, SAGE adaptively modulates the surrogate-gradient slope during training using the proposed uncertainty-driven entropy estimation while leaving the inference model unchanged. To investigate robustness across different temporal resolutions, all three surrogate-gradient configurations were evaluated under simulation time steps $T=\{1,2,4,8,12\}$. The corresponding classification accuracies are presented in Fig.~\ref{fig:surrogate_timestep_comparison}, where each subplot reports the Top-1 accuracy obtained on CIFAR-10, CIFAR-100, and ImageNet-200 for a fixed simulation time step.

\section{Discussion}
\label{sec:discussion}
\paragraph{Performance Analysis}
Table~\ref{tab:cifar_comparison} and Fig.~\ref{fig:surrogate_timestep_comparison} summarize the performance of the proposed SAGE framework under different surrogate-gradient formulations and simulation time steps. Across the evaluated datasets, SAGE consistently achieves competitive or superior classification accuracy compared with both the conventional fixed surrogate ($\alpha=4$) and the learnable surrogate formulation. Under the standard Spikformer setting ($T=4$), SAGE attains the highest Top-1 accuracy on both CIFAR-10 and CIFAR-100 while maintaining competitive performance on ImageNet-200. Furthermore, the temporal analysis across simulation time steps $T={1,2,4,8,12}$ shows that SAGE consistently achieves competitive or superior performance at lower simulation lengths ($T \leq 4$), with the largest improvement observed under the standard Spikformer setting of $T=4$. As the number of simulation time steps increases beyond $T=4$, the performance gap gradually narrows, and SAGE becomes comparable to the learnable surrogate formulation at $T=8$ and $T=12$, indicating that the benefits of uncertainty-guided surrogate adaptation are most pronounced in the low-latency regime. These results suggest that uncertainty-guided surrogate adaptation provides a robust optimization strategy that generalizes across multiple datasets and temporal settings while preserving the original Spikformer architecture.
\paragraph{Understanding SAGE}
The design of SAGE was motivated by the observation that different Spikformer blocks exhibit distinct attention dynamics throughout training. To identify a suitable uncertainty signal, we first analyzed the evolution of per-head attention statistics across training epochs. As illustrated in Fig.~\ref{fig:per_head_entropy_variability}, the standard deviation of the normalized attention entropy consistently revealed clear block-wise and head-wise variations, indicating that different transformer blocks experience varying levels of uncertainty during optimization. We also investigated alternative impurity measures, including Gini impurity~\cite{loh2011classification}, under multiple temperature scaling values ($T_{\mathrm{ent}}=\{0.25,0.5,1.0,2.0\}$), with the corresponding results provided in the Appendix Fig.~\ref{fig:gini_analysis}. While Gini impurity produced nearly saturated and highly uniform responses across blocks and temperatures, exhibiting limited temporal variation, the entropy-based formulation provided richer and more discriminative dynamics throughout training. Furthermore, using the standard deviation of per-head entropy, rather than the mean entropy, captures the dispersion among attention heads, directly reflecting the degree of disagreement within each transformer block. Since SAGE aims to adapt the surrogate gradient according to the consistency of attention behavior rather than its average confidence, entropy dispersion offers a more informative block-level uncertainty estimate for guiding surrogate-gradient adaptation.
\paragraph{Practical Implications}
Beyond the observed accuracy improvements, SAGE offers several practical advantages for spiking transformer optimization. Since the proposed framework operates exclusively during training, the forward inference graph remains identical to the original Spikformer architecture. Consequently, SAGE introduces no additional learnable parameters, preserves the original model size, and incurs no increase in inference latency or computational cost after training. The adaptive controller consists only of lightweight statistical operations, including entropy computation, exponential moving average updates, and a simple nonlinear mapping, contributing an average overhead of approximately $0.03$\,ms per iteration during training. As a result, the proposed method can be seamlessly integrated into existing surrogate-gradient training pipelines as a plug-and-play optimization strategy without requiring architectural modifications or changes to deployment, making it readily applicable to a broad range of SNNs and transformer-based SNN models.
\paragraph{Limitations and Future Work}
Although SAGE demonstrates consistent improvements on Spikformer across multiple image classification benchmarks, the current study focuses on attention-based spiking transformer architectures where uncertainty can be naturally estimated from self-attention distributions. Extending the proposed framework to convolutional SNNs or other non-transformer architectures will require alternative uncertainty measures, such as spike-rate or membrane-potential statistics. Furthermore, the present work adapts only the surrogate-gradient slope while keeping all neuron dynamics unchanged. Future work will investigate jointly adapting additional neuron parameters, including firing thresholds and membrane time constants, as well as learning more general uncertainty-aware controllers that can automatically optimize the surrogate-gradient behavior across different datasets, architectures, and simulation time steps.
\section{Conclusion}
\label{sec:conclusion}
We have presented SAGE, an uncertainty-aware surrogate-gradient adaptation framework for Spiking ViTs. By leveraging the dispersion of attention entropy across self-attention heads, SAGE dynamically modulates the surrogate-gradient slope during training while preserving the original network architecture, model parameters, and inference pipeline. Our experiments on CIFAR-10, CIFAR-100, and ImageNet-200 demonstrated that the proposed approach consistently improves or maintains competitive classification performance over fixed and learnable surrogate-gradient formulations, particularly in low-latency settings. Furthermore, SAGE introduces only negligible training-time overhead and no additional inference cost, making it a practical plug-and-play optimization strategy for spiking transformers. These results highlight the potential of uncertainty-guided surrogate optimization as an effective direction for improving the training of deep SNNs. 

\bibliographystyle{plainnat}
\bibliography{ref}

@article{ghosh2009spiking,
  title={Spiking neural networks},
  author={Ghosh-Dastidar, Samanwoy and Adeli, Hojjat},
  journal={International journal of neural systems},
  volume={19},
  number={04},
  pages={295--308},
  year={2009},
  publisher={World Scientific}
}

@article{tavanaei2019deep,
  title={Deep learning in spiking neural networks},
  author={Tavanaei, Amirhossein and Ghodrati, Masoud and Kheradpisheh, Saeed Reza and Masquelier, Timoth{\'e}e and Maida, Anthony},
  journal={Neural networks},
  volume={111},
  pages={47--63},
  year={2019},
  publisher={Elsevier}
}

@article{lee2016training,
  title={Training deep spiking neural networks using backpropagation},
  author={Lee, Jun H and Delbr{\"u}ck, Tobias and Pfeiffer, Michael},
  journal={Frontiers in neuroscience},
  volume={10},
  pages={508},
  year={2016},
  publisher={Frontiers Media}
}

@article{yao2023attention,
  title={Attention spiking neural networks},
  author={Yao, Man and Zhao, Guangshe and Zhang, Hengyu and Hu, Yifan and Deng, Lei and Tian, Yonghong and Xu, Bo and Li, Guoqi},
  journal={IEEE transactions on pattern analysis and machine intelligence},
  volume={45},
  number={8},
  pages={9393--9410},
  year={2023},
  publisher={IEEE}
}

@article{neftci2019surrogate,
  title={Surrogate gradient learning in spiking neural networks},
  author={Neftci, Emre O and Mostafa, Hesham and Zenke, Friedemann},
  journal={IEEE Signal Processing Magazine},
  volume={36},
  number={6},
  pages={51--63},
  year={2019},
  publisher={IEEE}
}

@inproceedings{yin2024dynamic,
  title={Dynamic spiking graph neural networks},
  author={Yin, Nan and Wang, Mengzhu and Chen, Zhenghan and De Masi, Giulia and Xiong, Huan and Gu, Bin},
  booktitle={Proceedings of the AAAI Conference on Artificial Intelligence},
  volume={38},
  number={15},
  pages={16495--16503},
  year={2024}
}

@inproceedings{zeng2025leveraging,
  title={Leveraging asynchronous spiking neural networks for ultra efficient event-based visual processing},
  author={Zeng, DingYi and Wang, Yuchen and Cao, Honglin and Liu, Wanlong and Xiao, Yichen and Chen, Wenyu and Zhang, Malu and Wang, Guoqing and Yang, Yang and others},
  booktitle={Proceedings of the AAAI Conference on Artificial Intelligence},
  volume={39},
  number={2},
  pages={1620--1628},
  year={2025}
}

@inproceedings{shi2024spikingresformer,
  title={Spikingresformer: Bridging resnet and vision transformer in spiking neural networks},
  author={Shi, Xinyu and Hao, Zecheng and Yu, Zhaofei},
  booktitle={Proceedings of the IEEE/CVF conference on computer vision and pattern recognition},
  pages={5610--5619},
  year={2024}
}

@inproceedings{wang2025spiking,
  title={Spiking vision transformer with saccadic attention},
  author={Wang, Shuai and Zhang, Dehao and Belatreche, Ammar and Xiao, Yichen and Liang, Yu and Shan, Yimeng and Sun, Qian and Zhang, Enqi and Zhang, Malu},
  booktitle={The 13th International Conference on Learning Representations},
  pages={72872--72893},
  year={2025},
  organization={International Conference on Learning Representations (ICLR)}
}

@article{zhou2022spikformer,
  title={Spikformer: When spiking neural network meets transformer},
  author={Zhou, Zhaokun and Zhu, Yuesheng and He, Chao and Wang, Yaowei and Yan, Shuicheng and Tian, Yonghong and Yuan, Li},
  journal={arXiv preprint arXiv:2209.15425},
  year={2022}
}

@article{yu2024spikingvit,
  title={Spikingvit: A multiscale spiking vision transformer model for event-based object detection},
  author={Yu, Lixing and Chen, Hanqi and Wang, Ziming and Zhan, Shaojie and Shao, Jiankun and Liu, Qingjie and Xu, Shu},
  journal={IEEE Transactions on Cognitive and Developmental Systems},
  volume={17},
  number={1},
  pages={130--146},
  year={2024},
  publisher={IEEE}
}

@inproceedings{xu2025hybrid,
  title={Hybrid Spiking Vision Transformer for Object Detection with Event Cameras},
  author={Xu, Qi and Deng, Jie and Shen, Jiangrong and Chen, Biwu and Tang, Huajin and Pan, Gang},
  booktitle={International Conference on Machine Learning},
  pages={69147--69159},
  year={2025},
  organization={PMLR}
}

@inproceedings{datta2025dynamic,
  title={Dynamic spikformer: Low-latency \& energy-efficient spiking neural networks with dynamic time steps for vision transformers},
  author={Datta, Gourav and Liu, Zeyu and Li, Anni and Beerel, Peter A},
  booktitle={ICASSP 2025-2025 IEEE International Conference on Acoustics, Speech and Signal Processing (ICASSP)},
  pages={1--5},
  year={2025},
  organization={IEEE}
}

@InProceedings{Nair_2026_CVPR,
    author    = {Nair, Kiran and Rizk, Rodrigue and Santosh, KC},
    title     = {A Spike-Gated Residual Unit for Information Flow Control in Transformers},
    booktitle = {Proceedings of the IEEE/CVF Conference on Computer Vision and Pattern Recognition (CVPR) Workshops},
    month     = {June},
    year      = {2026},
    pages     = {3469-3478}
}

@article{van2026adaptive,
  title={Adaptive surrogate gradients for sequential reinforcement learning in spiking neural networks},
  author={Van den Berghe, Korneel and Stroobants, Stein and Janapa Reddi, Vijay and De Croon, Guido},
  journal={Advances in Neural Information Processing Systems},
  volume={38},
  pages={147904--147926},
  year={2026}
}

@article{ai2025cross,
  title={A cross-layer residual spiking neural network with adaptive threshold leaky integrate-and-fire neuron and learnable surrogate gradient},
  author={Ai, Qingsong and Yang, Yingnan and Cai, Mincheng and Chen, Kun and Liu, Quan and Ma, Li},
  journal={Knowledge-Based Systems},
  volume={319},
  pages={113575},
  year={2025},
  publisher={Elsevier}
}

@article{hou2026adaptive,
  title={Adaptive and lightweight surrogate gradients: enhancing training efficiency of spiking neural networks},
  author={Hou, Kungjui and Wu, Kunlun and Zhou, Yongcheng},
  journal={Frontiers in Neuroscience},
  volume={20},
  pages={1795946},
  year={2026},
  publisher={Frontiers Media SA}
}

@article{jiang2025adaptive,
  title={Adaptive gradient learning for spiking neural networks by exploiting membrane potential dynamics},
  author={Jiang, Jiaqiang and Wang, Lei and Jiang, Runhao and Fan, Jing and Yan, Rui},
  journal={arXiv preprint arXiv:2505.11863},
  year={2025}
}

@inproceedings{sun2026optimization,
  title={Optimization method for surrogate function in spiking neural networks based on membrane potential distribution},
  author={Sun, Qi and Cao, Zhen and Geng, Kaige and Zhang, Ziyi and Hou, Biao},
  booktitle={Proceedings of the AAAI Conference on Artificial Intelligence},
  volume={40},
  number={30},
  pages={25718--25726},
  year={2026}
}

@inproceedings{jiang2026ds,
  title={Ds-atgo: dual-stage synergistic learning via forward adaptive threshold and backward gradient optimization for spiking neural networks},
  author={Jiang, Jiaqiang and Xu, Wenfeng and Fan, Jing and Yan, Rui},
  booktitle={Proceedings of the AAAI Conference on Artificial Intelligence},
  volume={40},
  number={3},
  pages={1855--1863},
  year={2026}
}

@article{stanojevic2024high,
  title={High-performance deep spiking neural networks with 0.3 spikes per neuron},
  author={Stanojevic, Ana and Wo{\'z}niak, Stanis{\l}aw and Bellec, Guillaume and Cherubini, Giovanni and Pantazi, Angeliki and Gerstner, Wulfram},
  journal={Nature Communications},
  volume={15},
  number={1},
  pages={6793},
  year={2024},
  publisher={Nature Publishing Group UK London}
}

@article{perez2021sparse,
  title={Sparse spiking gradient descent},
  author={Perez-Nieves, Nicolas and Goodman, Dan},
  journal={Advances in Neural Information Processing Systems},
  volume={34},
  pages={11795--11808},
  year={2021}
}

@article{che2022differentiable,
  title={Differentiable hierarchical and surrogate gradient search for spiking neural networks},
  author={Che, Kaiwei and Leng, Luziwei and Zhang, Kaixuan and Zhang, Jianguo and Meng, Qinghu and Cheng, Jie and Guo, Qinghai and Liao, Jianxing},
  journal={Advances in Neural Information Processing Systems},
  volume={35},
  pages={24975--24990},
  year={2022}
}

@inproceedings{wang2023adaptive,
  title={Adaptive smoothing gradient learning for spiking neural networks},
  author={Wang, Ziming and Jiang, Runhao and Lian, Shuang and Yan, Rui and Tang, Huajin},
  booktitle={International conference on machine learning},
  pages={35798--35816},
  year={2023},
  organization={PMLR}
}

@article{guo2024take,
  title={Take a shortcut back: Mitigating the gradient vanishing for training spiking neural networks},
  author={Guo, Yufei and Chen, Yuanpei and Hao, Zecheng and Peng, Weihang and Jie, Zhou and Zhang, Yuhan and Liu, Xiaode and Ma, Zhe},
  journal={Advances in Neural Information Processing Systems},
  volume={37},
  pages={24849--24867},
  year={2024}
}

@inproceedings{branchini2023causal,
  title={Causal entropy optimization},
  author={Branchini, Nicola and Aglietti, Virginia and Dhir, Neil and Damoulas, Theodoros},
  booktitle={International Conference on Artificial Intelligence and Statistics},
  pages={8586--8605},
  year={2023},
  organization={PMLR}
}

@article{cheng2026group,
  title={Group Entropy-Controlled Policy Optimization},
  author={Cheng, Guangran and Lyu, Chengqi and Gao, Songyang and Zhang, Wenwei and Chen, Kai},
  journal={arXiv preprint arXiv:2607.16850},
  year={2026}
}

@inproceedings{qiao2021uncertainty,
  title={Uncertainty-guided model generalization to unseen domains},
  author={Qiao, Fengchun and Peng, Xi},
  booktitle={Proceedings of the IEEE/CVF conference on computer vision and pattern recognition},
  pages={6790--6800},
  year={2021}
}

@inproceedings{lai2024uncertainty,
  title={Uncertainty-guided never-ending learning to drive},
  author={Lai, Lei and Ohn-Bar, Eshed and Arora, Sanjay and Yi, John Seon Keun},
  booktitle={Proceedings of the IEEE/CVF Conference on Computer Vision and Pattern Recognition},
  pages={15088--15098},
  year={2024}
}

@inproceedings{moon2020confidence,
  title={Confidence-aware learning for deep neural networks},
  author={Moon, Jooyoung and Kim, Jihyo and Shin, Younghak and Hwang, Sangheum},
  booktitle={international conference on machine learning},
  pages={7034--7044},
  year={2020},
  organization={PMLR}
}

@article{lv2026confidence,
  title={Confidence-Aware With Prototype Alignment for Partial Multi-label Learning},
  author={Lv, Weijun and Chen, Yu and Fang, Xiaozhao and Zhu, Xuhuan and Wen, Jie and Zhou, Guoxu and Chan, Sixian},
  journal={Advances in Neural Information Processing Systems},
  volume={38},
  pages={170397--170416},
  year={2026}
}

@article{vaswani2017attention,
  title={Attention is all you need},
  author={Vaswani, Ashish and Shazeer, Noam and Parmar, Niki and Uszkoreit, Jakob and Jones, Llion and Gomez, Aidan N and Kaiser, {\L}ukasz and Polosukhin, Illia},
  journal={Advances in neural information processing systems},
  volume={30},
  year={2017}
}

@article{MAASS19971659,
title = {Networks of spiking neurons: The third generation of neural network models},
journal = {Neural Networks},
volume = {10},
number = {9},
pages = {1659-1671},
year = {1997},
issn = {0893-6080},
doi = {https://doi.org/10.1016/S0893-6080(97)00011-7},
url = {https://www.sciencedirect.com/science/article/pii/S0893608097000117},
author = {Wolfgang Maass},
}

@inproceedings{kundu2021hire,
  title={Hire-snn: Harnessing the inherent robustness of energy-efficient deep spiking neural networks by training with crafted input noise},
  author={Kundu, Souvik and Pedram, Massoud and Beerel, Peter A},
  booktitle={Proceedings of the IEEE/CVF international conference on computer vision},
  pages={5209--5218},
  year={2021}
}

@article{roy2019towards,
  title={Towards spike-based machine intelligence with neuromorphic computing},
  author={Roy, Kaushik and Jaiswal, Akhilesh and Panda, Priyadarshini},
  journal={Nature},
  volume={575},
  number={7784},
  pages={607--617},
  year={2019},
  publisher={Nature Publishing Group UK London}
}

@article{wu2018spatio,
  title={Spatio-temporal backpropagation for training high-performance spiking neural networks},
  author={Wu, Yujie and Deng, Lei and Li, Guoqi and Zhu, Jun and Shi, Luping},
  journal={Frontiers in neuroscience},
  volume={12},
  pages={331},
  year={2018},
  publisher={Frontiers Media SA}
}

@inproceedings{han2020rmp,
  title={Rmp-snn: Residual membrane potential neuron for enabling deeper high-accuracy and low-latency spiking neural network},
  author={Han, Bing and Srinivasan, Gopalakrishnan and Roy, Kaushik},
  booktitle={Proceedings of the IEEE/CVF conference on computer vision and pattern recognition},
  pages={13558--13567},
  year={2020}
}

@article{bellec2018long,
  title={Long short-term memory and learning-to-learn in networks of spiking neurons},
  author={Bellec, Guillaume and Salaj, Darjan and Subramoney, Anand and Legenstein, Robert and Maass, Wolfgang},
  journal={Advances in neural information processing systems},
  volume={31},
  year={2018}
}

@article{zhou2026advancing,
  title={Advancing Direct Training for Spiking Neural Networks with Circulate-Firing Neurons and Learnable Gradients},
  author={Zhou, Feifan and Wei, Xiang and Liu, Yang and Yu, Qiang},
  journal={arXiv preprint arXiv:2605.27412},
  year={2026}
}

@inproceedings{horowitz20141,
  title={1.1 computing's energy problem (and what we can do about it)},
  author={Horowitz, Mark},
  booktitle={2014 IEEE international solid-state circuits conference digest of technical papers (ISSCC)},
  pages={10--14},
  year={2014},
  organization={IEEE}
}

@inproceedings{zheng2021going,
  title={Going deeper with directly-trained larger spiking neural networks},
  author={Zheng, Hanle and Wu, Yujie and Deng, Lei and Hu, Yifan and Li, Guoqi},
  booktitle={Proceedings of the AAAI conference on artificial intelligence},
  volume={35},
  number={12},
  pages={11062--11070},
  year={2021}
}

@article{zenke2018superspike,
  title={Superspike: Supervised learning in multilayer spiking neural networks},
  author={Zenke, Friedemann and Ganguli, Surya},
  journal={Neural computation},
  volume={30},
  number={6},
  pages={1514--1541},
  year={2018},
  publisher={MIT Press One Rogers Street, Cambridge, MA 02142-1209, USA journals-info~…}
}

@article{shrestha2018slayer,
  title={Slayer: Spike layer error reassignment in time},
  author={Shrestha, Sumit B and Orchard, Garrick},
  journal={Advances in neural information processing systems},
  volume={31},
  year={2018}
}

@article{gawlikowski2023survey,
  title={A survey of uncertainty in deep neural networks: J. Gawlikowski et al.},
  author={Gawlikowski, Jakob and Tassi, Cedrique Rovile Njieutcheu and Ali, Mohsin and Lee, Jongseok and Humt, Matthias and Feng, Jianxiang and Kruspe, Anna and Triebel, Rudolph and Jung, Peter and Roscher, Ribana and others},
  journal={Artificial intelligence review},
  volume={56},
  number={Suppl 1},
  pages={1513--1589},
  year={2023},
  publisher={Springer}
}

@article{rathi2020enabling,
  title={Enabling deep spiking neural networks with hybrid conversion and spike timing dependent backpropagation},
  author={Rathi, Nitin and Srinivasan, Gopalakrishnan and Panda, Priyadarshini and Roy, Kaushik},
  journal={arXiv preprint arXiv:2005.01807},
  year={2020}
}

@article{rathi2020diet,
  title={Diet-snn: Direct input encoding with leakage and threshold optimization in deep spiking neural networks},
  author={Rathi, Nitin and Roy, Kaushik},
  journal={arXiv preprint arXiv:2008.03658},
  year={2020}
}

@article{deng2022temporal,
  title={Temporal efficient training of spiking neural network via gradient re-weighting},
  author={Deng, Shikuang and Li, Yuhang and Zhang, Shanghang and Gu, Shi},
  journal={arXiv preprint arXiv:2202.11946},
  year={2022}
}

@article{bu2023optimal,
  title={Optimal ANN-SNN conversion for high-accuracy and ultra-low-latency spiking neural networks},
  author={Bu, Tong and Fang, Wei and Ding, Jianhao and Dai, PengLin and Yu, Zhaofei and Huang, Tiejun},
  journal={arXiv preprint arXiv:2303.04347},
  year={2023}
}

@article{krizhevsky2009learning,
  title={Learning multiple layers of features from tiny images},
  author={Krizhevsky, Alex and Hinton, Geoffrey and others},
  year={2009},
  publisher={Toronto, ON, Canada}
}

@inproceedings{deng2009imagenet,
  title={Imagenet: A large-scale hierarchical image database},
  author={Deng, Jia and Dong, Wei and Socher, Richard and Li, Li-Jia and Li, Kai and Fei-Fei, Li},
  booktitle={2009 IEEE conference on computer vision and pattern recognition},
  pages={248--255},
  year={2009},
  organization={Ieee}
}

@article{shannon1948mathematical,
  title={A mathematical theory of communication},
  author={Shannon, Claude Elwood},
  journal={The Bell system technical journal},
  volume={27},
  number={3},
  pages={379--423},
  year={1948},
  publisher={Nokia Bell Labs}
}

@article{loh2011classification,
  title={Classification and regression trees},
  author={Loh, Wei-Yin},
  journal={Wiley interdisciplinary reviews: data mining and knowledge discovery},
  volume={1},
  number={1},
  pages={14--23},
  year={2011},
  publisher={Wiley Online Library}
}

\clearpage
\appendix

\section{Additional Analysis of Uncertainty Measures}
To investigate suitable uncertainty measures for surrogate-gradient adaptation, we compared normalized attention entropy with Gini impurity computed from the attention distributions. Figure~\ref{fig:gini_analysis} illustrates the evolution of Gini impurity across the four Spikformer blocks under different temperature scaling values ($T_{\mathrm{ent}}=\{0.25,0.5,1.0,2.0\}$). Compared with the entropy-based analysis presented in the main paper, Gini impurity exhibits substantially smaller temporal variation and remains nearly saturated throughout training, providing limited discrimination between transformer blocks. Consequently, SAGE adopts the standard deviation of normalized attention entropy as the uncertainty signal for adaptive surrogate-gradient modulation.

\begin{figure}[H]
\centering

\includegraphics[width=0.48\textwidth]{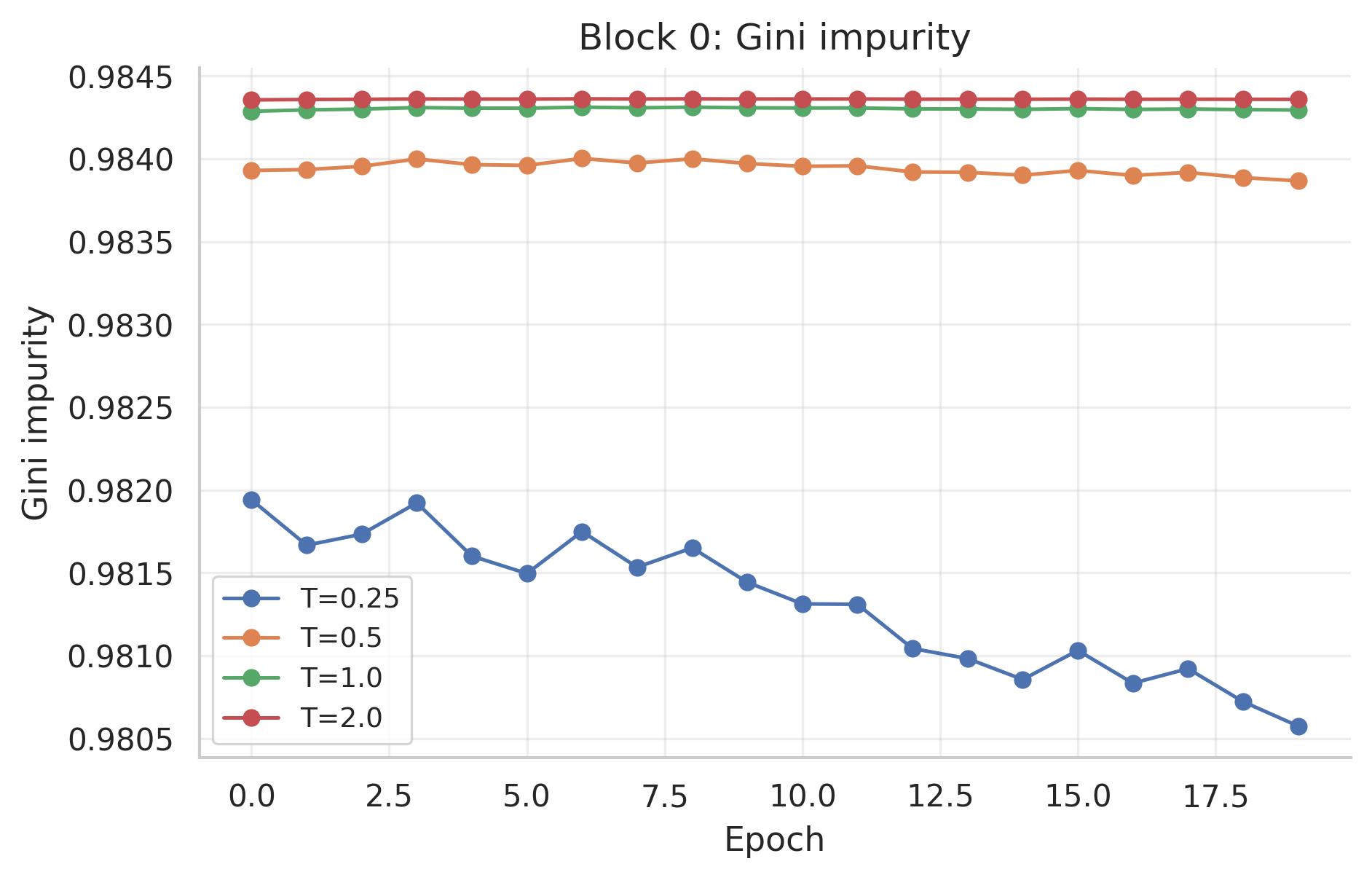}
\hfill
\includegraphics[width=0.48\textwidth]{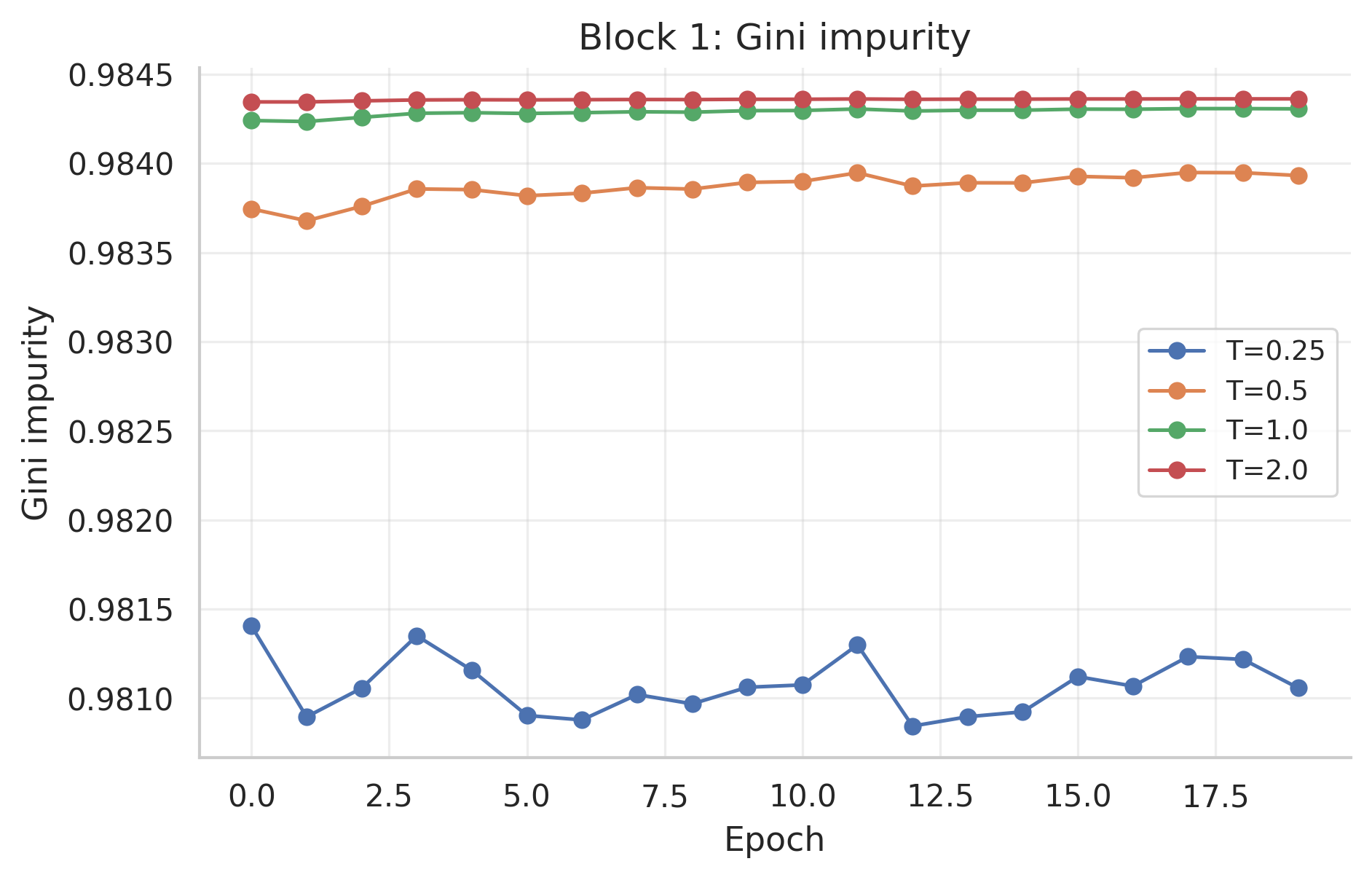}

\vspace{0.3em}

\includegraphics[width=0.48\textwidth]{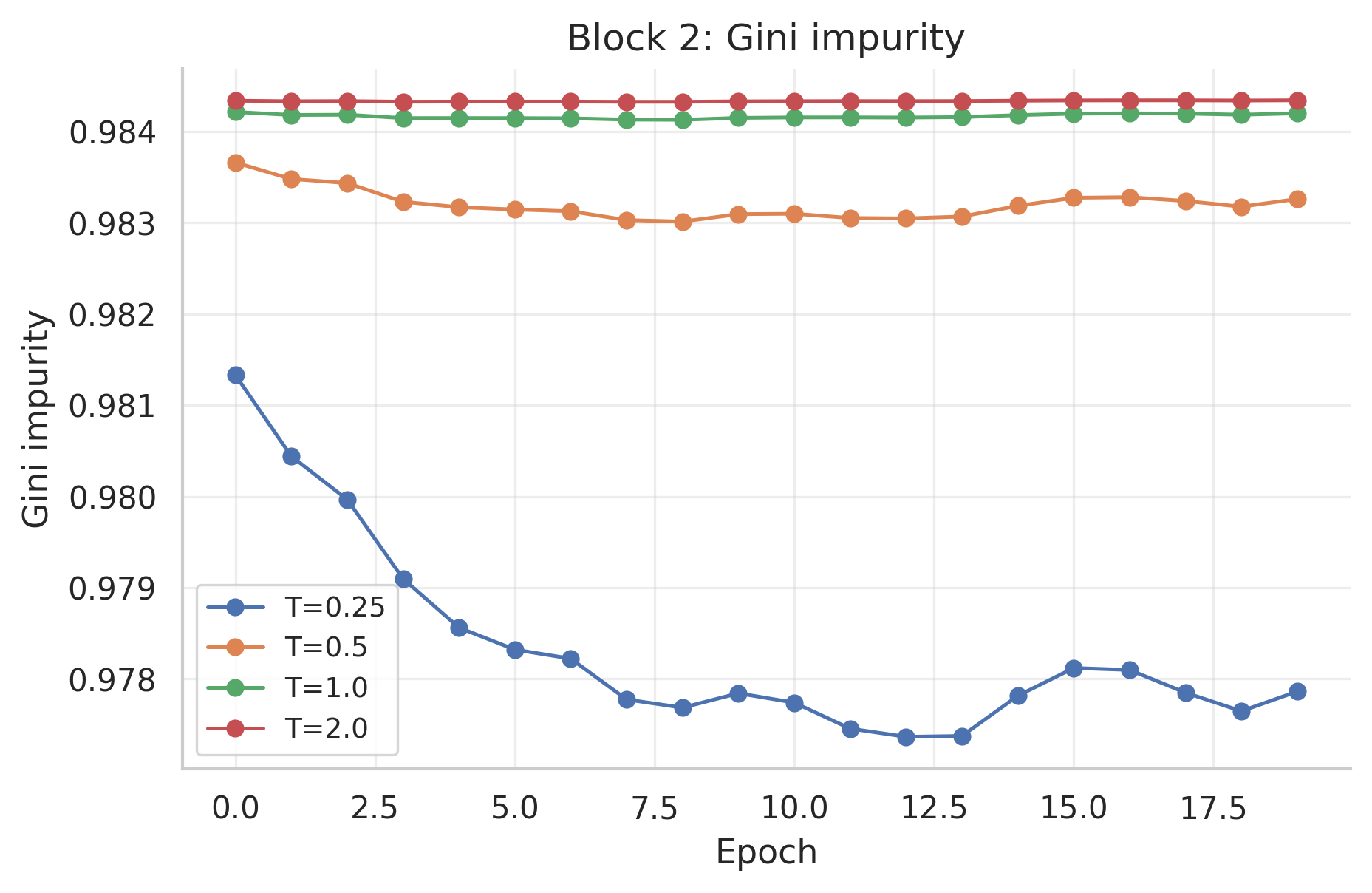}
\hfill
\includegraphics[width=0.48\textwidth]{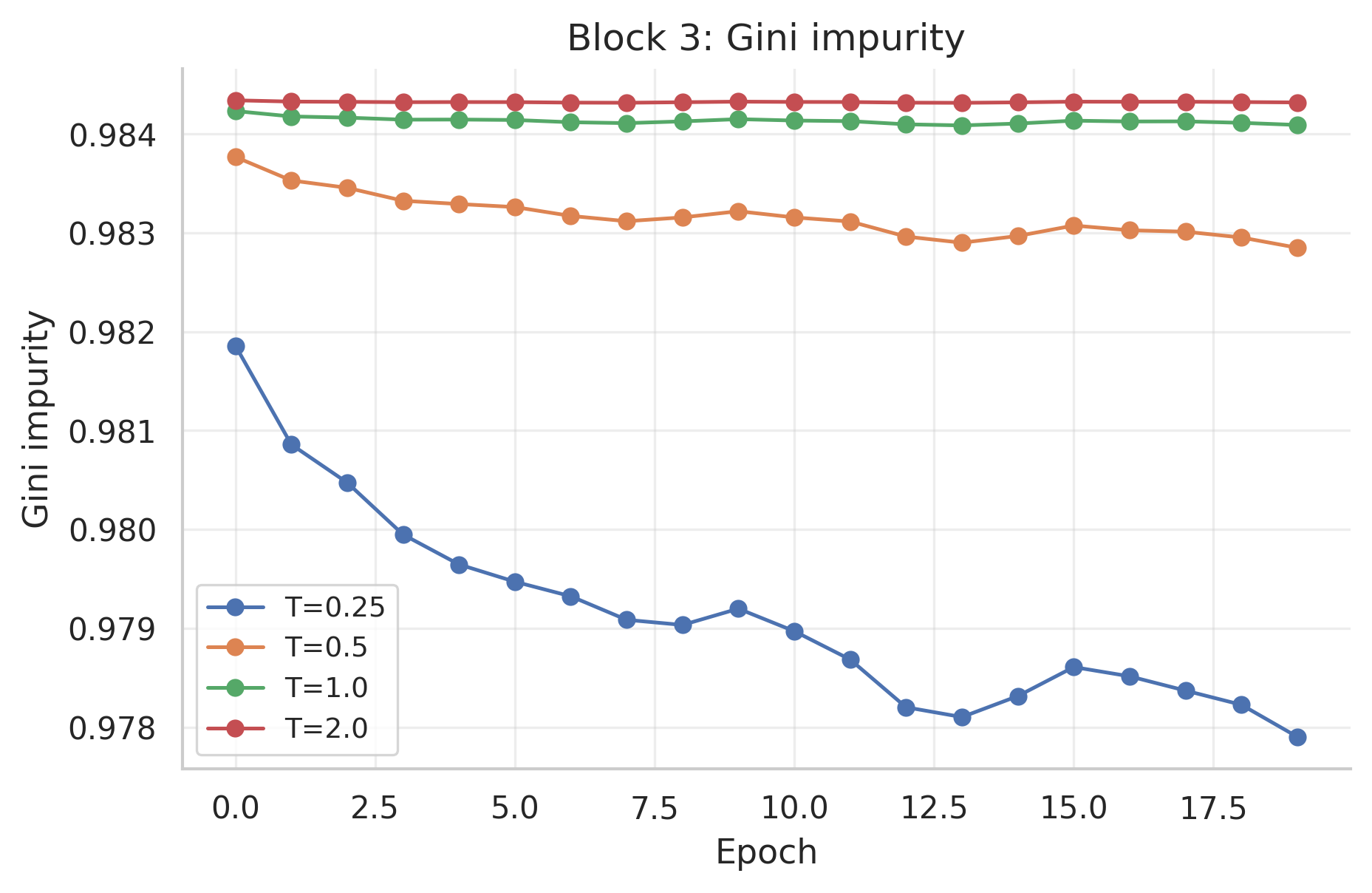}

\caption{Evolution of Gini impurity across the four Spikformer transformer blocks during training under different attention-temperature values ($T_{\mathrm{ent}}=0.25$, $0.5$, $1.0$, and $2.0$). Across all blocks, Gini impurity remains highly saturated with only minor temporal fluctuations, indicating limited sensitivity to changes in attention uncertainty.}

\label{fig:gini_analysis}
\end{figure}

To determine an appropriate entropy temperature for SAGE, we performed a one-factor-at-a-time (OFAT) sensitivity analysis by evaluating the evolution of the standard deviation of normalized attention entropy under four temperature values ($T_{\mathrm{ent}}=\{0.25,0.5,1.0,2.0\}$). Figure~\ref{fig:temperature_analysis} summarizes the results across the four Spikformer transformer blocks. Lower entropy temperatures preserve substantially greater entropy dispersion, providing a richer uncertainty signal for the adaptive controller, whereas larger temperatures progressively smooth the attention distribution and reduce the available variability. Consequently, $T_{\mathrm{ent}}=0.25$ was selected for all experiments.
\begin{figure*}[!t]
\centering

\includegraphics[width=0.48\textwidth]{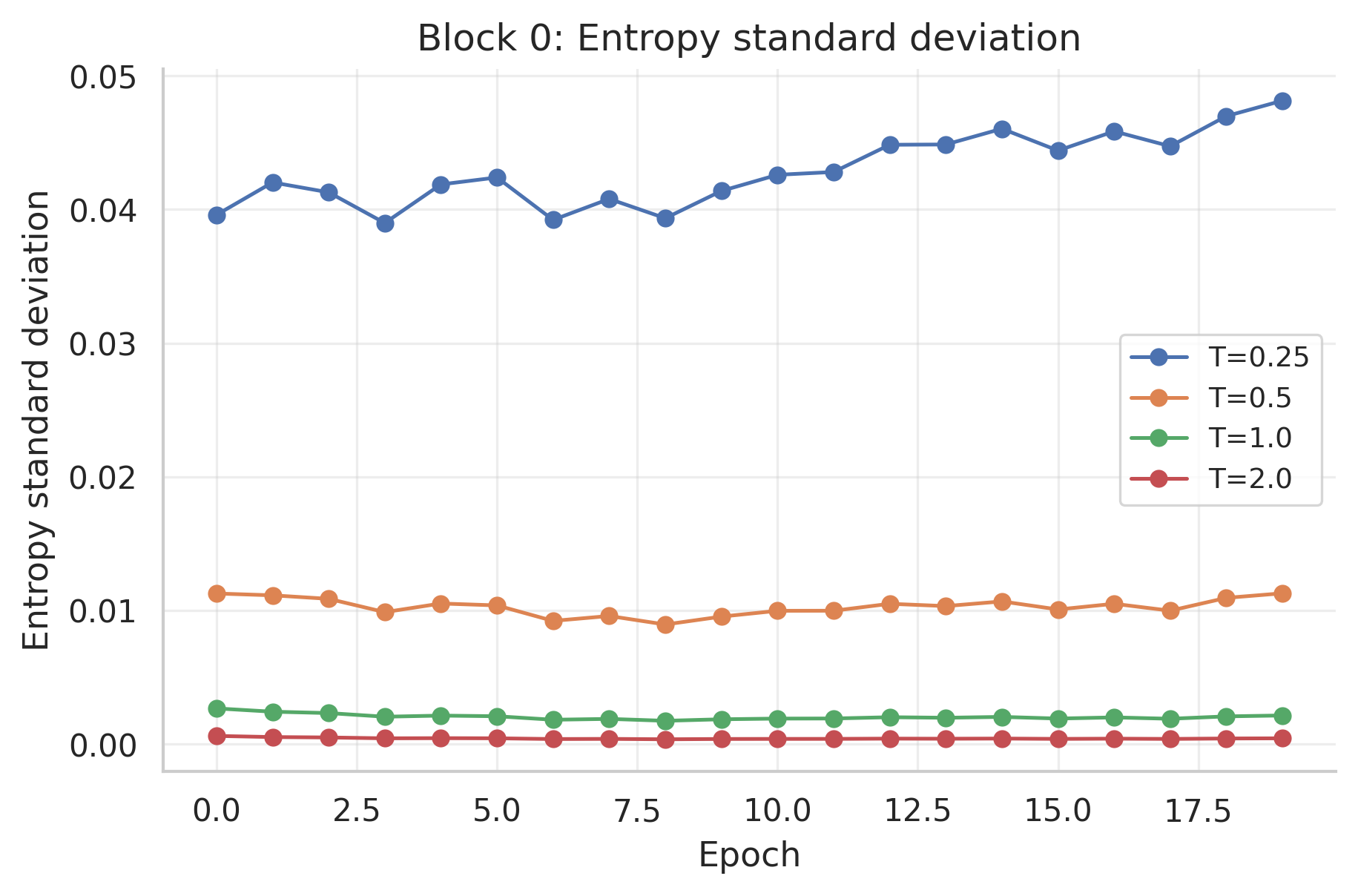}
\hfill
\includegraphics[width=0.48\textwidth]{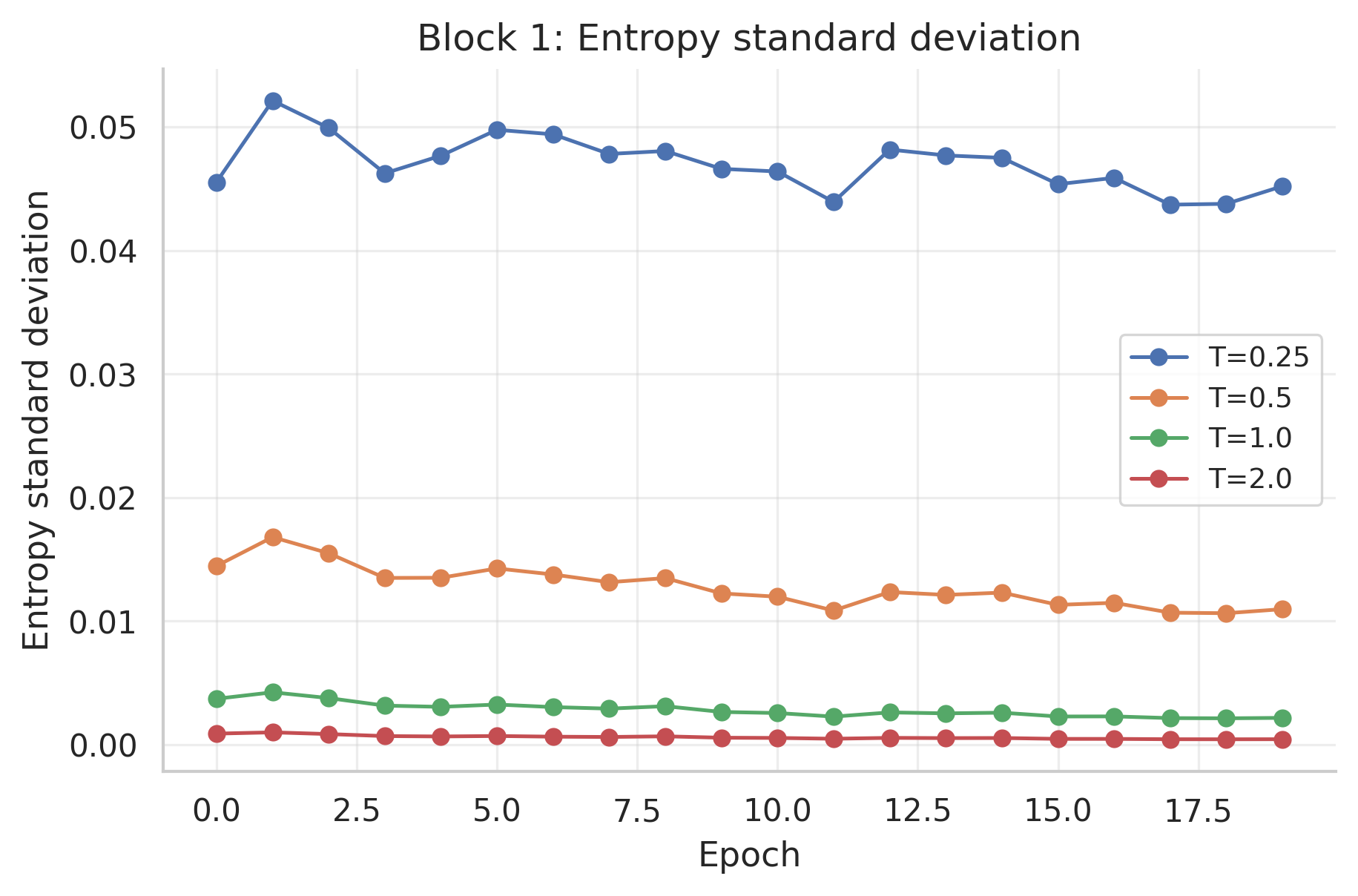}

\vspace{0.3em}

\includegraphics[width=0.48\textwidth]{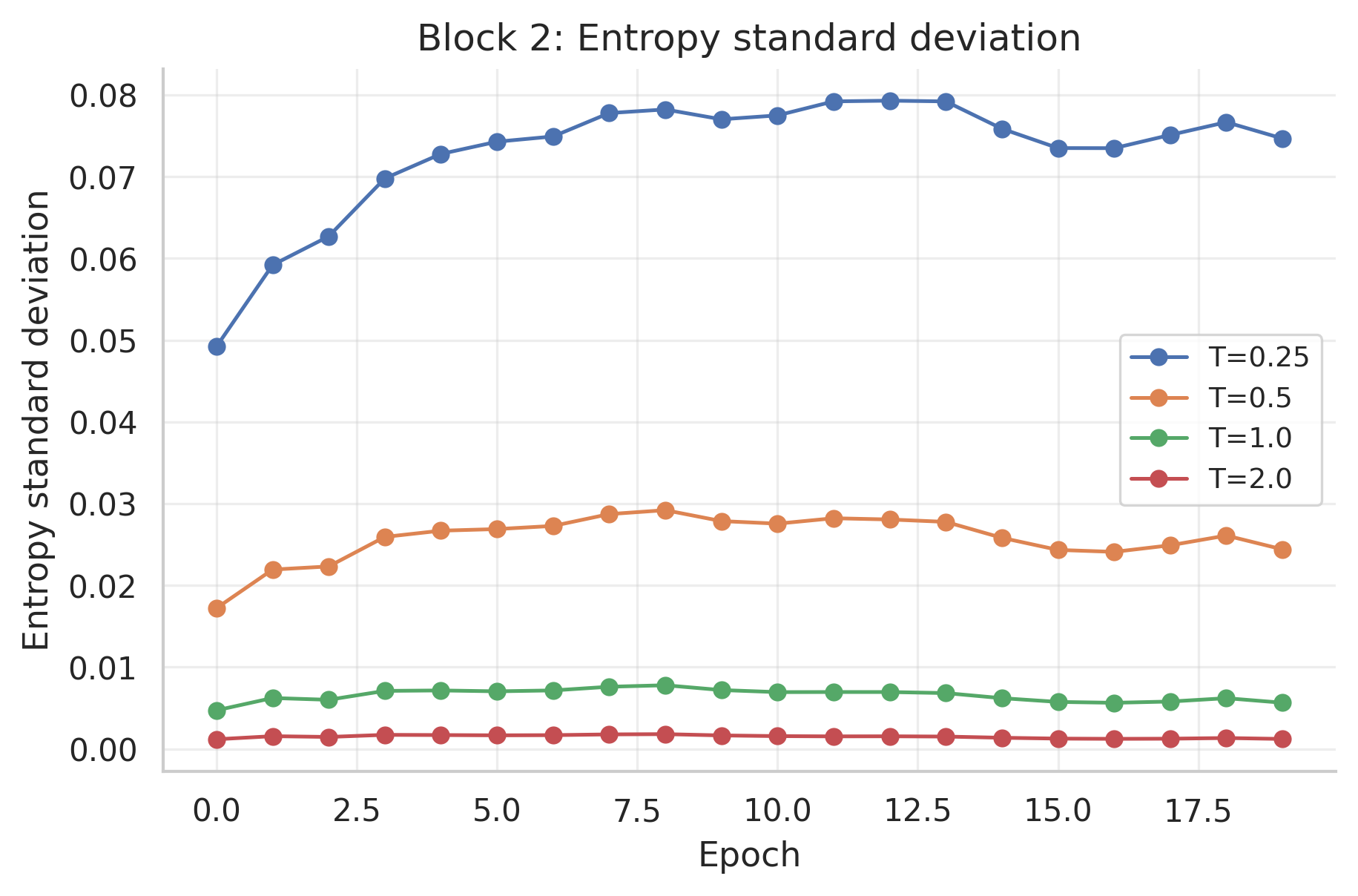}
\hfill
\includegraphics[width=0.48\textwidth]{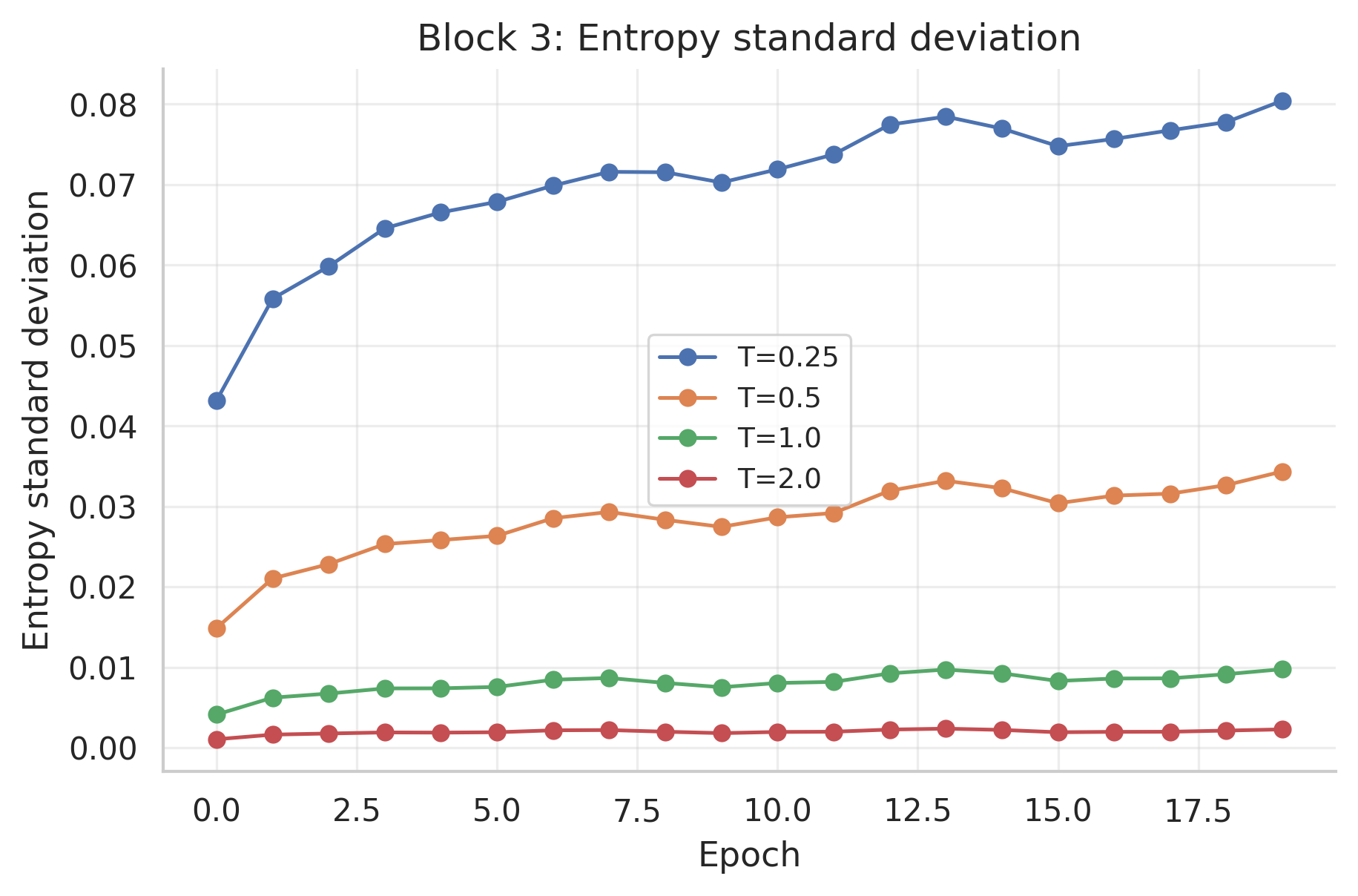}

\caption{Sensitivity analysis of the entropy temperature. Evolution of the standard deviation of normalized attention entropy across the four Spikformer blocks during training under different entropy temperatures ($T_{\mathrm{ent}}=0.25$, $0.5$, $1.0$, and $2.0$). Based on this analysis, $T_{\mathrm{ent}}=0.25$ was adopted throughout all experiments.}

\label{fig:temperature_analysis}
\end{figure*}

\begin{algorithm}[t]
\caption{SAGE Training for Spikformer}
\label{alg:umsen}
\scriptsize
\begin{algorithmic}[1]
\REQUIRE Training set $\mathcal{D}$, Spikformer model $f_{\theta}$ with $B$ SSA blocks, baseline surrogate slope $\alpha_0=4.0$
\REQUIRE Entropy temperature $\tau_e=0.25$, EMA factor $\beta=0.95$, dead-zone threshold $\delta=0.25$
\REQUIRE Slope bounds $\alpha_{\min}=3.0$, $\alpha_{\max}=5.0$
\STATE Initialize one controller per SSA block: $\mathcal{C}_1,\ldots,\mathcal{C}_B$
\STATE Initialize EMA dispersion $m_b$, running mean $\mu_b$, running variance $\sigma_b^2$, and slope $\alpha_b=\alpha_0$ for each block $b$
\FOR{each epoch $e=1,\ldots,E$}
    \FOR{each mini-batch $(x,y)\in\mathcal{D}$}
        \IF{$e=1$ or training step is in warm-up}
            \STATE Set $\alpha_b \leftarrow \alpha_0$ for all blocks
        \ENDIF

        \STATE Apply current block-wise surrogate slopes $\{\alpha_b\}_{b=1}^{B}$ to all LIF nodes inside each transformer block
        \STATE Forward pass through Spikformer:
        \[
            \hat{y}=f_{\theta}(x)
        \]

        \FOR{each SSA block $b=1,\ldots,B$}
            \STATE Extract detached raw attention scores:
            \[
                A_b = (Q_bK_b^\top)\cdot s
            \]
            \STATE Compute temperature-scaled attention probabilities:
            \[
                P_b=\operatorname{softmax}\left(\frac{A_b}{\tau_e}\right)
            \]
            \STATE Compute normalized attention entropy:
            \[
                H_b=-\frac{1}{\log N}\sum_{j=1}^{N}P_{b,j}\log(P_{b,j}+\epsilon)
            \]
            \STATE Average entropy over time, batch, and token dimensions to obtain per-head entropy:
            \[
                h_b=\operatorname{mean}_{T,B,N}(H_b)
            \]
            \STATE Compute raw entropy dispersion across heads:
            \[
                d_b=\operatorname{std}(h_b)
            \]
            \STATE Update EMA dispersion:
            \[
                m_b \leftarrow \beta m_b + (1-\beta)d_b
            \]
            \STATE Update running mean and variance of $m_b$
            \STATE Compute temporal normalized uncertainty:
            \[
                z_b=\frac{m_b-\mu_b}{\sigma_b+\epsilon}
            \]
        \ENDFOR

        \STATE Center normalized uncertainty across blocks:
        \[
            \tilde{z}_b=z_b-\frac{1}{B}\sum_{k=1}^{B}z_k
        \]

        \FOR{each block $b=1,\ldots,B$}
            \IF{$|\tilde{z}_b| < \delta$}
                \STATE $\alpha_b \leftarrow \alpha_0$
            \ELSE
                \STATE $\alpha_b \leftarrow \alpha_0 + 0.5\tanh(\tilde{z}_b)$
            \ENDIF
            \STATE Clamp:
            \[
                \alpha_b \leftarrow \operatorname{clip}(\alpha_b,\alpha_{\min},\alpha_{\max})
            \]
        \ENDFOR

        \STATE Compute loss $\mathcal{L}(\hat{y},y)$ using the official training criterion
        \STATE Backpropagate using surrogate gradients parameterized by the active $\alpha_b$
        \STATE Update model parameters $\theta$ with the optimizer
        \STATE Reset spiking neuron states
    \ENDFOR
\ENDFOR
\end{algorithmic}
\end{algorithm}

\end{document}